\documentclass[12pt]{article}
\usepackage{graphicx,graphics,psfrag,epsf}
\usepackage{enumerate}
\usepackage{natbib}
\usepackage{url} 

\usepackage{amsmath,amssymb,amsfonts,amsthm,bbm}
\usepackage{color}

\usepackage{titlesec}
\usepackage[titletoc,toc,title]{appendix}

\usepackage{titling}
\pretitle{\begin{center}\huge}
\posttitle{\par\end{center}\vspace{\baselineskip}}
\preauthor{\normalfont\normalsize\begin{center}\begin{tabular}[t]{c}}
\postauthor{\end{tabular}\end{center}\vspace{\baselineskip}}

\newcommand{\EE}{\mathbb{E}}

\newcommand{\PP}{\mathbb{P}}

\newcommand{\RR}{\mathbb{R}}

         \def\cF{{\cal  F}}

          \def\cX{{\cal  X}}

\renewcommand{\hat}{\widehat}

\def \heps     {\hat{\heps}}

\DeclareMathOperator*{\argmin}{argmin}

\def\IF{\mbox{IF}}
\def\CVF{\mbox{CVF}}

\def\SS{\mbox{SS}}
\def\LS{\mbox{LS}}
\def\GS{\mbox{GS}}

\newtheoremstyle{mytheoremstyle} 
    {\topsep}                    
    {\topsep}                    
    {\normalfont}                   
    {}                           
    {\bfseries}                   
    {.}                          
    {.5em}                       
    {}  

\theoremstyle{mytheoremstyle}
\newtheorem{theorem}{Theorem}
\newtheorem{Definition}{Definition}

\newtheorem{lemma}{Lemma}
\newtheorem{proposition}{Proposition}
\newtheorem{remark}{Remark}
\newtheorem{corollary}{Corollary}

\newtheorem{condition}{Condition}

\makeatletter
\def\singlespace{\def\baselinestretch{2}\@normalsize}

\makeatletter
\def\singlespace{\def\baselinestretch{1}\@normalsize}

\renewcommand{\baselinestretch}{1.2}

\title{\textsc{Privacy-preserving parametric inference: a case for robust statistics} }
\author{ Marco Avella-Medina \thanks{Columbia University, Department of Statistics, New York, NY, USA, email: \texttt{marco.avella@columbia.edu}.
The author is grateful for the financial support of the Swiss National Science Foundation and would like to thank Roy Welsch for many helpful discussions.}}
\date{November 20, 2019 \small{(First version: May 15,  2018)}}
\begin{document}
\def\spacingset#1{\renewcommand{\baselinestretch}%
{#1}\small\normalsize} \spacingset{1}

\maketitle

\begin{abstract}
Differential privacy is a cryptographically-motivated approach to privacy that  has become a very active field of research over the last decade in theoretical computer science and machine learning. In this paradigm one assumes there is a trusted curator who holds the data of individuals in a database and the goal of privacy is to simultaneously protect individual data while allowing the release of global characteristics of the database. In this setting we introduce a general framework for parametric inference with differential privacy guarantees. We first obtain differentially private estimators based on bounded influence M-estimators by leveraging their gross-error sensitivity in the calibration of a noise term added to them in order to ensure privacy.  We then  show how a similar construction can also be applied to construct differentially private test statistics analogous to the Wald, score and likelihood ratio tests. We provide statistical guarantees for all our proposals via an asymptotic analysis. An interesting consequence of our results is to further clarify the connection between differential privacy and robust statistics. In particular, we demonstrate that differential privacy is a weaker stability requirement than infinitesimal robustness, and show that robust M-estimators can be easily randomized in order to guarantee both differential privacy and robustness towards the presence of contaminated data. We illustrate our results both on simulated and real data.
\end{abstract}

\newpage
 \spacingset{1.2} 
\section{Introduction}

Differential privacy is a cryptographically-motivated approach to  privacy which has become a very active field of research over the last decade in theoretical computer science and machine learning \citep{dworkandroth2014}. In this paradigm one assumes there is a trusted curator who holds the data of individuals in a database that might for instance be constituted by $n$ individual rows. The goal of privacy is to simultaneously protect every individual row while releasing global characteristics of the database.  Differential privacy provides such guarantees in the context of remote access query systems where the data analysts do not get to see the actual data, but can ask a server for the output of some statistical model.  Here the trusted curator processes the queries of the user and releases noisy versions of the desired output in order to protect individual level data. 

The interest in remote access systems was prompted by the recognition of fundamental failures of anonymization approaches. Indeed, it is now well acknowledged that releasing data sets without obvious individual identifiers such as names and home addresses are not sufficient to preserve privacy. The problem with such approaches is that an ill-intentioned user might be able to link the anonymized data with external non anonymous data. Hence auxiliary information could help intruders break anonymization and learn sensitive information.  One prominent example of privacy breach is the  de-anonymization of a Massachusetts hospital discharge database by joining it with with a public voter database in \cite{sweeney1997}.  In fact combining anonymization with sanitization techniques such as adding noise to the dataset directly or removing certain entries of the data matrix are also fundamentally flawed \citep{narayananandshmatikov2008}. On the other hand,  differential privacy  provides a rigorous mathematical framework to the notion of privacy by guaranteeing protection against identity attacks regardless of the auxiliary information that may be available to the attackers. 
This is achieved by requiring that the output of a query does not change too much if we add or remove any individual from the data set. Therefore the user cannot learn much about any individual data record from the output requested.

There is now a large body of literature in this topic and recent work has sought to link differential privacy to statistical problems by developing privacy-preserving algorithms for empirical risk minimization, point estimation and density estimation \citep{dworkandlei2009,wassermanandzhou2010,smith2011,chaudhurietal2011, bassilyetal2014}. 
Despite the numerous developments made in the area of differential privacy since the seminal work of \cite{dworketal2006}, one can argue that their practical utility in applied scientific work  is very limited by the lack of broad guidelines for statistical inference. In particular, there are  no generic procedures for performing statistical hypothesis testing for general parametric models which arguably constitutes one of the cornerstones of a statisticians data analysis toolbox.

\subsection{Our contribution}

The basic idea of our  work is to introduce differentially private algorithms leveraging tools from robust statistics. In particular, we use the Gaussian mechanism studied in the differential privacy literature in combination with robust statistics sensitivity measures. At a high level, this mechanism provides a generic way to release a noisy version of a statistical query, where the noise level is carefully calibrated to ensure privacy. For this purpose, appropriate notions of sensitivity have been studied in the computer science literature. By focusing on the class of parametric M-estimators, we show that the well studied statistics notion of sensitivity  given by the influence function can also be used to calibrate the Gaussian mechanism. This logic extends to tests derived from M-estimators since their sensitivity can also be understood via the influence function. 

To the best of our knowledge, our work is the first one to provide a systematic treatment of  estimation and hypothesis testing with differential privacy guarantees in the context of  general parametric models. The main contributions of this paper are the following:
\begin{enumerate}
\item[(a)] We introduce a general class of differentially private parametric estimators under mild conditions. Our estimators are computationally efficient and can be tuned to trade-off  statistical efficiency and robustness.  
\item[(b)]  We propose differentially private counterparts of the Wald, score and likelihood ratio tests for parametric models. Our proposals are by construction robust in a contamination neighborhood of the assumed generative model and  are easily constructed from readily available statistics. 
\item[(c)]  We further clarify the connections between differential privacy and robust statistics by showing that the influence function can be used to bound the smooth sensitivity of \cite{nissimetal2007}. It follows that bounded-influence estimators can naturally be used to construct differentially private estimators. The converse is not true as our analysis shows that one can construct differentially private estimators that asymptotically do not have a bounded influence function.
\end{enumerate}

\subsection{Related work}

The notion of  differential privacy is very similar to the intuitive one of robustness in statistics. The latter requires that  no small portion of the data should influence too much a statistical analysis \citep{huberandronchetti2009, hampeletal1986,belsleyetal2005,maronnaetal2006}. This  connection has been noticed in previous works that have shown how to construct differentially private robust estimators. In particular, the estimators of \citep{dworkandlei2009, smith2011, lei2011, chaudhuriandhsu2012} are the most closely related to ours since they all provide differentially private parametric estimators building on M-estimators and establish statistical convergence rates.  However, our construction compares favorably to previous proposals in many regards. Our estimators preserve the optimal parametric $\sqrt{n}$-consistency, and hence our privacy guarantees do not come at the expense of slower statistical rates of convergence as in \citep{dworkandlei2009,lei2011}. Furthermore we do not assume a known diameter of the parameter space as in \cite{smith2011}. Our construction is inspired by the univariate estimator of \cite{chaudhuriandhsu2012} which is in general computationally inefficient as it requires the computation of the smooth sensitivity defined in Section 2.2. We broaden the scope of their technique to general multivariate M-estimators and more importantly, we overcome the computational barrier intrinsic to their method by showing that the empirical influence function can be used in the noise calibration of the Gaussian mechanism. %
 There are however other possible approaches to construct differentially private estimators. Here we discuss three popular alternatives that have been explored in the literature. 
 
The first approach seeks to design a mechanism to release differentially private data instead of constructing new estimators. This can be achieved by constructing a differentially private density estimator such as a perturbed histogram of the data. Once such a density estimator is available it can  be used to either sample private data \citep{wassermanandzhou2010} or to construct a weighted differentially private objective function  for empirical risk minimization \citep{lei2011}. Although the latter approach leads to better rates of convergence for parametric estimation, they remain slow and have a bad dimension dependence $\max\{1/\sqrt{n},(\sqrt{\log n}/n)^{2/(2+p)}\}$, where $n$ is the sample size and $p$ is the dimension of the estimated parameter. Indeed,  this approach suffers from the curse of dimensionality since it relies on the computation of multivariate density estimators. Interestingly, a somehow related approach for releasing synthetic data existed  in the statistics literature  prior to the advent of differential privacy \citep{rubin1993,reiter2002,reiter2005} and consequently also  lacks formal theoretical privacy guarantees. 

A second approach consist of releasing estimators that are defined as the minimizers of a perturbed objective function. Representative work in this direction includes  \cite{chaudhuriandmonteleoni2008} in the context of penalized logistic regression,  \cite{chaudhurietal2011} in the general learning problem of empirical risk minimization and \cite{kieferetal2012} in a high dimensional regression setting. A related idea to perturbing the objective function is to  is to run a stochastic gradient descent algorithm where at each iteration update step an appropriately scaled noise term is added to the gradient in order to ensure privacy. This idea was used for example by   \cite{rajkumarandargawal2012} in the context of multiparty classification, \cite{bassilyetal2014} in the general learning setting of empirical risk minimization and  \cite{wangetal2015sgd} for Bayesian learning. Although the potential applicability of these two perturbation approaches to a wide variety of models makes them appealing, it remains unclear how to construct test statistics in these settings.


A third alternative approach  is to draw samples from a well suited probability distribution. The exponential mechanism of \cite{mcsherryandtalwar2007} is a main example of a general method for achieving $(\varepsilon,0)$-differential privacy via random sampling. This idea leads naturally to  connections with posterior sampling in Bayesian statistics. Some papers exploring these ideas include \cite{chaudhuriandhsu2012}  and \cite{dimitrakakisetal2014,dimitrakakisetal2017}. See also \cite{fouldsetal2016} for a broader discussion of different mechanism for constructing privacy preserving Bayesian methods. Bayesian approaches that provide differentially private posterior distributions seem to be naturally amenable for the construction of confidence intervals and test statistics, as explored in \cite{liu2016}.  However it does not seem obvious to us how to use Bayesian privacy preserving results such \cite{dimitrakakisetal2014,dimitrakakisetal2017, fouldsetal2016} in order to provide analogue constructions to ours for estimation and testing. Interestingly, in this line of work  the typical regularity conditions required on the likelihood and  prior distribution are reminiscent of the regularity conditions required in frequentists setups as discussed below in Section 3.1.

The literature on   hypothesis testing with differential privacy guarantees is much more recent and limited than the one focusing on estimation.  A few  papers tackling this problem are the work of  \citep{uhleretal2013, wangetal2015, gaboardietal2016} who consider differentially private chi-squared tests and \citep{sheffet2017, barrientosetal2019} who provide differentially private t-tests for the regression coefficients of a linear regression model. 
Our approach is more broadly applicable since it extends to general parametric models and also weakens the distributional assumptions required by existing differentially private estimation and testing techniques. Roughly speaking, this is due to the fact that our M-estimators are robust by construction and will therefore have an associated bounded influence function.  It is worth noting that the latter property automatically guarantees gradient Lipschitz conditions that have previously been assumed for differentially private empirical risk minimizers \citep{chaudhurietal2011,bassilyetal2014}. 
After submitting the first version of this paper,  we have noticed some interesting new developments on differentially private inference in the work of  \citep{awanandslavkovic2018, awanandslavkovic2019, canonneetal2019a,  canonneetal2019b}.

  One interesting new development in the literature that we do not cover in this work is  local differential privacy. This new paradigm accounts for settings in which even the statistician collecting the data is not trusted \citep{duchietal2018}. This scenario leads to slower minimax optimal convergence rates of estimation for many important problems including mean estimation and logistic regression. \cite{sheffet2018} seems to be the first work exploring the problem of hypothesis testing under local differential privacy.

\subsection{Organization of the paper}

In Section 2 we overview some key background notions   from differential privacy and robust statistics that we use throughout the paper. In Section 3 we introduce our technique for constructing differentially private estimators and study their theoretical properties. In Section 4 we show how to further extend our construction to test functionals in order to perform differentially private hypothesis testing using M-estimators. In Section 5 we illustrate the numerical performance of our methods in both synthetic and real data. We conclude our paper in Section 6 with a discussion of our results and future research directions. We relegated to the Appendix all the proofs and some  auxiliary results and discussions.

\textit{Notation:}  $\| V \|$ denotes either euclidean norm if $V\in\mathbb{R}^N$ or its induced operator norm if $V\in\mathbb{R}^{N\times N}$. The smallest and largest eigenvalues of a matrix $A$ are denoted by $\lambda_{\min}(A)$ and $\lambda_{\max}(A)$. For two probability measures $P$ and $Q$,  the notation $d_{\infty}(P,Q)$ and $d_{TV}(P,Q)$  stand for sup-norm (Kolmogorov-Smirnov) and total variation distance. We reserve calligraphic letters such as $\mathcal{S}$ for sets and denote their cardinality by $|\mathcal{S}|$.  For two sets of  $\mathcal{S}$ and $\mathcal{S}'$ of the same size, we denote their Hamming distance by $d_H(\mathcal{S},\mathcal{S}'):=|\mathcal{S}\setminus \mathcal{S}'|=|\mathcal{S}'\setminus \mathcal{S}|$.

\section{Preliminaries}

Let us first review some important background concepts from differential privacy,  robust statistics and the M-estimation framework for parametric models.

\subsection{Differential privacy}

 Consider a database consisting of a  set of data points $\mathcal D=\{x_1,\dots,x_n\}\in\mathfrak{X}^n$, where $\mathfrak{X}\subset \mathbb R^m$ is some data space. We also use the notation $\mathcal{D}(F_n)$ to emphasize that $\mathcal{D}$ can be viewed as a data set associated with an empirical distribution $F_n$ induced by $\{x_1,\dots,x_n\}$. Differential privacy seeks to release useful information from the data set while protecting information about any individual data entry.

\begin{Definition}
A randomized function $A(\mathcal{D})$ is  $(\varepsilon,\delta)$--\emph{differentially private} if for all pairs of databases $(\mathcal{D},\mathcal{D}')$ with $d_H(\mathcal{D},\mathcal{D}')=1$ and all measurable subsets of outputs $\mathcal O$:
$$\mathbb{P}(A(\mathcal{D})\in\mathcal O)\leq e^{\varepsilon}\mathbb{P}(A(\mathcal{D}')\in\mathcal O)+\delta. $$
\end{Definition}

Intuitively, $(\varepsilon,0)$-differential privacy ensures that for every run of  algorithm $A$ the output is almost equally likely to be observed on every neighboring database.  This condition is relaxed by $(\varepsilon,\delta)$-differential privacy since it allows that given a random output $O$ drawn from $ A(\mathcal D)$, it may be possible to find a database $\mathcal D'$ such that $O$ is more likely to be produced on $\mathcal D'$ that it is when the database is $\mathcal D$. However such an event will be extremely unlikely. In both cases the similarity is defined by the factor $e^\varepsilon$ while the probability of deviating from this similarity is $\delta$.

 The magnitude of the  privacy parameters $(\varepsilon,\delta)$ are typically considered to be quite different. We are particularly interested in negligible values of $\delta$ that are smaller than the inverse of any polynomial in the size $n$ of the database. The rational behind this requirement  is that values of $\delta$ of the order of $\|x\|_1$, for some vector values database $x$, are problematic since they ``preserve privacy" while allowing to publish the complete records of a small number of individuals in the database. On the other hand, the privacy parameter $\varepsilon$ is typically thought of as a moderately small constant and in fact  ``the nature of privacy guarantees with differing but small epsilons are quite similar'' \cite[p.25]{dworkandroth2014}. Indeed,  failing to be $(\varepsilon,0)$-differentially private for some large $\varepsilon$ (i.e. $\varepsilon=10$) is just saying that there is a least a pair of neighboring datasets and an output $\mathcal O$ for which the ratio of probabilities of observing $\mathcal O$ conditioned on the database being $\mathcal D$ or $\mathcal D'$ is large.

One can  naturally wonder how to compare two differentially private algorithms $A_1$ and $A_2$ with different associated  privacy parameters $(\epsilon_1,\delta_2)$ and $(\epsilon_2,\delta_2)$. It seems natural to prefer the algorithm that ensures the smallest privacy loss incurred by observing some output i.e. $\log\big(\mathbb{P}(A(\mathcal{D})\in\mathcal O)/\mathbb{P}(A(\mathcal{D}')\in\mathcal O)\big)$. Since we only consider negligible $\delta_1$ and $\delta_2$, the privacy loss will be approximately proportional to the privacy parameter $\varepsilon$. One  could consequently prefer the algorithm with the smallest parameter $\varepsilon$ even though we say that roughly speaking ``all small epsilons are alike'' \citep[p.24]{dworkandroth2014}.

Differential privacy enjoys certain appealing properties that facilitates the design and analysis of complicated algorithms with privacy guarantees. Perhaps the two most important ones are that  $(\varepsilon,\delta)$-differential privacy is immune to post-processing and that combining two differentially private algorithms preserves differential privacy. More precisely, if $A$ is $(\varepsilon,\delta)$-differentially private, then the composition of any data independent mapping $f$ with $A$ is also $(\varepsilon,\delta)$-differentially private. In other words, releasing $f(A(\mathcal D))$ for any $\mathcal D$ still  guarantees  $(\varepsilon,\delta)$-differential privacy.
Furthermore, if we have two algorithms $A_1$ and $A_2$ with different associated  privacy parameters $(\epsilon_1,\delta_2)$- and $(\epsilon_2,\delta_2)$, then releasing the outputs of $A_1(\mathcal D)$ and $A_2(\mathcal D)$ guarantees $(\varepsilon_1+\varepsilon_2,\delta_1+\delta_2)$-differential privacy.  We refer interested readers to \citep[Chapters 2--3]{dworkandroth2014} for a more extensive discussion of the concepts presented in this subsection.

\subsection{Constructing differentially private algorithms}

A general and very popular technique for constructing differentially private algorithms is the Laplace mechanism, which consists of adding some well calibrated noise to the output of a standard query \citep{dworketal2006}. This procedure relies on suitable  notions of sensitivity of the function that is queried. All the following definitions of sensitivity are standard in the differential privacy literature and are typically defined with respect to the $L_1$ norm. We will instead use the Euclidean norm for the construction of our estimators as explained below.

\begin{Definition}
The \emph{global sensitivity}  of a function $\varphi:\mathfrak{X}^n\to\mathbb{R}^p$ is
$$\GS(\varphi):= \sup_{\mathcal{D},\mathcal{D}'} \Big\{\|\varphi(\mathcal{D})-\varphi(\mathcal{D}')\|: d_H(\mathcal{D},\mathcal{D}')=1 \Big\}.$$
The \emph{local sensitivity}  of a function $\varphi:\mathfrak{X}^n\to\mathbb{R}^p$ at a data set $\mathcal{D}\in\mathfrak{X}^n$ is
$$\LS(\varphi,\mathcal{D}):= \sup_{\mathcal{D}'} \Big\{\|\varphi(\mathcal{D})-\varphi(\mathcal{D}')\|: d_H(\mathcal{D},\mathcal{D}')=1 \Big\}.$$
For $\xi>0$, the $\xi$--\emph{smooth sensitivity} of $\varphi$ at $\mathcal{D}$ is
$$\SS_\xi(\varphi,\mathcal{D}):=\sup_{\mathcal{D}'}\Big\{e^{-\xi d_H(\mathcal{D},\mathcal{D}')}\LS(\varphi,\mathcal{D}'): \mathcal{D}'\in\mathfrak{X}^n \Big\}. $$
\end{Definition}

We are now ready to describe two versions of the Laplace mechanism using the above sensitivity notions defined with respect to the $L_1$ norm.  Denote by Lap$(b)$ a scaled symmetric Laplace distribution with density function $h_b(x)=\frac{1}{2b}\exp(-\frac{|x|}{b})$ and let Lap$_p(b)$ be the multivariate distribution obtained from $p$ independent and identically distributed $X_j\sim \mbox{Lap}(b)$ for $j=1,\dots,p.$  A key idea introduced in the seminal paper \cite{dworketal2006} is that for a function $f:\mathfrak{X}^n\to\mathbb{R}^p$ and an input database $\mathcal{D}$, one can simply compute $f(\mathcal{D})$ and then generate an independent noise term $U\sim \mbox{Lap}_p(\GS(f)/\varepsilon)$ in order to construct a $(\varepsilon,0)$-differentially private output $f(\mathcal{D})+U$. A related idea introduced by \cite{nissimetal2007} is to calibrate the noise using the smooth sensitivity instead of the local sensitivity. These authors showed that provided $\xi= \frac{\varepsilon}{4(p+2\log(2/\delta))}$ and $\tilde{U}\sim \mbox{Lap}_p(\SS_\xi(f)/\varepsilon)$, then the output $f(\mathcal{D})+\tilde{U}$ is $(\varepsilon,\delta)$-differentially private. Our proposals will build on the latter idea for the construction of private estimation and inferential procedures for parametric models.

 We would like to point out that the different notions of sensitivity introduced in  Definition 2 are usually defined with respect to the $L_1$ norm. We chose to  instead present these definitions in terms of the Euclidean distance as they are more naturally connected to well studied concepts in robust statistics. In particular, it leads to connections with the standard way of presenting the notion of gross-error sensitivity in robust statistics and the related problem of optimal B-robust estimation \cite[Chapter 4]{hampeletal1986}.  Because we focus on sensitivities with respective to the Euclidean metric, our construction follows the same logic of the Laplace mechanism, but naturally replaces the noise distribution with an appropriately scaled normal random variable as proposed in \cite{nissimetal2007}. In this case the output $f(\mathcal{D})+\tilde{U}$ is $(\varepsilon,\delta)$-differentially private if $\tilde{U}\sim N_p(0,\sigma^2 I)$ where $\sigma=5\sqrt{2\log(2/\delta)}\SS_\xi(f)/\varepsilon$ and  $\xi= \frac{\varepsilon}{4(p+2\log(2/\delta))}$. For obvious reasons the resulting procedure has been called the Gaussian mechanism in \cite{dworkandroth2014}.  As we were completing the revision of the current manuscript we noticed that \cite{caietal2019} have also worked with this mechanism for the derivation of the optimal statistical minimax rates of convergence for parametric estimation under $(\varepsilon,\delta)$-differential privacy.

\subsection{Robust statistics}

Robust statistics provides a theoretical framework that allows to take into account that models are only idealized approximations of reality and develops methods that give results that are stable  when slight deviations from the stochastic assumptions of the model occur. Book-length expositions on the topic include \citep{huber1981, huberandronchetti2009, hampeletal1986, maronnaetal2006}.  We will focus on the infinitesimal robustness approach that considers the impact of moderate distributional deviations from ideal models on a statistical procedure \citep{hampeletal1986}. In this setting the statistics of interest are viewed as functionals of the underlying distribution and the influence function is the key tool used to assess the robustness of a statistical functional. 
\begin{Definition}
Given  a measurable space $\mathfrak{Z}$, a  distribution space $\mathfrak{F}$, a parameter space $\Theta\subset\mathbb R^p$ and a functional $T:\mathfrak{F}\mapsto\Theta$, the influence function of $T$ at a point $z\in\mathfrak{Z}$ for a distribution $F$ is defined as 
$$\mathrm{IF}(z;T,F):=\lim_{t\to 0+}\frac{T(F_t)-T(F)}{t}, $$
where $F_t=(1-t)F+t\Delta_z$ and $\Delta_z$ is a mass point at $z$. 
\end{Definition}
The influence function has the heuristic interpretation of describing the effect of an infinitesimal contamination at the point $z$ on the estimate, standardized by the mass of contamination. Furthermore, if a statistical functional $T(F)$ is sufficiently regular, a von Mises expansion \citep{vonmises1947,hampel1974, hampeletal1986} yields
\begin{equation}
\label{vonMises}
T(G) = T(F)+\int \IF(z; T,F)\mathrm{d}(G-F)(z) +o\big(d_\infty(G,F)\big).
\end{equation}
  Considering the  approximation \eqref{vonMises} over a neighborhood of the form $\mathfrak{F}_t=\{F^{(t)}| F^{(t)}=(1-t)F+ t G, ~$G$ \mbox{ an arbitrary distribution}\}$,  we see that the influence function can be used to linearize the asymptotic bias in a neighborhood of the idealized model $F$. Therefore, a statistical functional with bounded influence function is robust in the sense that it will have a bounded approximate bias in a neighborhood of $F$. 
 A related notion of robustness is the gross-error sensitivity which measures the worst case value of the influence function.
\begin{Definition}
The \emph{gross-error sensitivity} of a functional $T:\mathfrak{F}\to \Theta$ at the distribution $F$ is
$$\gamma(T,F):=\sup_{x\in\mathfrak{X}}\|\mathrm{IF}(x;T,F)\|. $$
\end{Definition}
Clearly if the space $\mathfrak{X}$ is unbounded, the gross-error sensitivity of $T$ will be infinite unless its influence function is uniformly bounded. In Sections 3 and 4 we will show how to use the robust statistics tools described here in the construction of differentially private estimators and tests.
 
\subsection{M-estimators for parametric models} 
M-estimators are a  simple class of estimators that is appealing from a robust statistics perspective and  constitute a very general approach to parametric inference \citep{huber1964,huberandronchetti2009}. They will be the focus of the rest  of  this paper. An M-estimator  $\hat\theta=T(F_n)$ of $\theta_0\in\Theta\subset\mathbb{R}^p$ is defined as a solution to 
$$\sum_{i=1}^n\Psi(x_i,T(F_n))=0,$$
where $\Psi:\mathbb{R}^m\times\Theta\to\mathbb{R}^p$,  $x_1,\dots,x_n\in\mathbb{R}^m$ are independent identically distributed according to $F$ and $F_n$ denotes the empirical distribution function. This class of estimators is a strict generalization of the class of regular maximum likelihood estimators. Assuming that $T(F)=\theta_0$ and some mild conditions \citep[Ch. 6]{huberandronchetti2009}, as $n\to\infty$ they are asymptotically normally distributed as
\begin{equation*}
\sqrt{n}(T(F_n)-\theta_0)\to_{d} N(0,V(T,F)),
\end{equation*}
where $V(T,F)=\mathbb{E}_F[\IF(X;T,F)\IF(X;T,F)^T]$ and $\mathbb{E}_F[\IF(X;T,F)]=0$. Furthermore, their influence function is 
 \begin{equation}
 \label{IF}
 \IF(x;T,F)= \Big(M(T,F)\Big)^{-1}\Psi(x,T(F)),
 \end{equation}
 where $M(T,F)=-\mathbb{E}_F[\dot{\Psi}(X,T(F))]=-\frac{\partial}{\partial \theta}\mathbb{E}_F[\Psi(X,\theta)]\big|_{\theta=\theta_0}$. Therefore M-estimators defined by bounded functions  $\Psi$ are said to be infinitesimally robust  since their influence function is bounded and by \eqref{vonMises} their asymptotic bias will also be bounded for small amounts of contamination.

\section{Differentially private estimation}

\subsection{Assumptions}

In the following we allow $\Psi$ to depend on $n$, but we do not stress it in the notation to make it less cumbersome. Here  are the main conditions required in our analysis:

\begin{condition}
\label{cond:boundedpsi}
The function $\Psi(x,\theta)$ is differentiable with respect to $\theta$ almost everywhere for all $x\in\mathfrak X$, and we denote this derivative by $\dot{\Psi}(x,\theta)$. Furthermore, for all $\theta\in \Theta$ there exists constants $K_n,L_n>0$ such that 
$$\sup_{x\in\mathfrak{X}} \|\Psi(x,\theta)\|\leq K_n ~~~\mbox{ and }~~~ \sup_{x\in\mathfrak{X}} \|\dot{\Psi}(x,\theta)\|\leq L_n.$$
\end{condition}
\begin{condition}
\label{cond:hessian}
The matrix $M_F=M(T,F)=-\mathbb{E}_F[\dot{\Psi}(X,T(F))]$ is positive definite at the generative distribution $F=F_{\theta_0}$. Furthermore the space of data sets $\mathfrak X^{n}$ is such that for all empirical distributions $G_n\in\{G| \mathcal{D}(G)\in \mathfrak X^n\}$ with $n\geq N_0$ we have that $0< b\leq\lambda_{\min}(M_{G_n})\leq\lambda_{\max}(M_{G_n})\leq B <\infty$.
\end{condition}
\begin{condition}
\label{cond:smoothness}
There exist $r_1>0$, $r_2>0$, $r_3>0$, $C_1$ and $C_2>0$ such that 
$$\|\mathbb{E}_{F_n}[\dot{\Psi}(X,\theta)]-\mathbb{E}_{G_n}[\dot{\Psi}(X,\theta)]\|\leq C_1 d_{\infty}(F_n,G_n) \mbox{ and } $$
$$\|\mathbb{E}_{F_n}[\dot{\Psi}(X,\theta)]-\mathbb{E}_{F_n}[\dot{\Psi}(X,T(F_n))]\|\leq C_2\|T(F_n)-\theta\| $$
whenever $d_{\infty}(F_n,G_n)\leq r_1$, $\|\theta-T(F_n)\|\leq r_2$ and $\|\theta-\theta_0\|\leq r_3$.
\end{condition}

Condition \ref{cond:boundedpsi} requires  $\Psi$ and $\dot{\Psi}$ to be uniformly bounded in $\mathfrak{X}$ by some potentially diverging constants $K_n$ and $L_n$. The case $K_n=K<\infty$ is particularly appealing from a robust statistics perspective as it guarantees that the resulting M-estimators has a bounded influence function. If additionally $L_n=L<\infty$, then the resulting M-estimator will also be second order infinitesimally robust as defined in \cite{lavecchiaetal2012} and will have a bounded change of variance function; see \cite{hampeletal1981} and  our Appendix \ref{AppendixCVF} for more details.
Condition \ref{cond:hessian} restricts the space of data sets to one where some minimal regularity conditions on the Jacobian of $\Psi$ hold.  Similar assumptions are usually required to guarantee the asymptotic normality and Fr{\'e}chet differentiability of M-estimators, see for example  \cite{huber1967}, \citep[Corollary 6.7]{huberandronchetti2009} and \cite{clarke1986}. Our assumptions are stronger in order to guarantee that $M_{G_n}$ is invertible and hence that the empirical influence function is computable. Even though such requirements are not always explicitly stated, common statistical practice implicitly assumes them when computing  estimated asymptotic variances with plug-in formulas. In a standard linear regression setting these conditions boil down to assuming that the design matrix is full rank.  Even such a  seemingly harmless condition seems stronger in the differential privacy context. Indeed, it might not be checkable by the users and one would like to have such a guarantee to hold over all possible configurations of the data. One possible way of tackling this problem is to  let the algorithm  halt with an output ``No Reply'' when this assumption fails  \citep{dworkandlei2009, avellamedinaandbrunel2019}. 
 Condition \ref{cond:smoothness} is a smoothness condition on $\dot\Psi$ at $F_n$, similar to Condition 4 in \cite{chaudhuriandhsu2012}. It is a technical assumption used when upper bounding the smooth sensitivity by the gross-error sensitivity. The constants $C_1$ and $C_2$ are effectively Lipschitz constants.

 We would like to highlight that since the differential privacy paradigm assumes a remote access query framework where the user does not get to see the data, in principle it is not immediate that the user will be able to check basic features of the data e.g. whether the design matrix is full rank before performing an analysis. This is a serious limitation of this paradigm as it more generally prevents users from performing exploratory data analysis before fitting a model and it is also unclear how to do model checking and run diagnostics on fitted models. One would have to develop differentially private analogues of the whole data analysis pipeline in order to allow a data analyst to perform rigorous statistical analysis. An interesting recent development in this direction in a regression setting is the work of \cite{chenetal2018}.

\subsection{A general construction}

Let us now introduce our mechanism for constructing differentially private M-estimators. Given a statistical M-functional $T$, we propose the randomized estimator
\begin{equation}
\label{Gaussianmechanism1}
A_T(F_n):=T(F_n)+\gamma(T,F_n)\frac{5\sqrt{2\log(n)\log(2/\delta)}}{\varepsilon n}Z,
\end{equation}
where $Z$ is a $p$ dimensional  standard normal  random variable. The intuition behind our proposal is simple: the gross-error sensitivity $\gamma(T,F_n)$  should  be roughly of the same order as the smooth sensitivity.  Therefore multiplying it by $\sqrt{\log(n)}$ will guarantee that it upper bounds the smooth sensitivity. This in turn suffices to guarantee $(\varepsilon,\delta)$-differential privacy.    From a computational perspective, using the empirical gross-error-sensitivity is much more appealing than computing the exact smooth sensitivity. Indeed, the former can be  further upper bounded in practice using the empirical influence function  whereas the latter can be very difficult to compute in general as discussed in \cite{nissimetal2007}.

\begin{theorem}
\label{thm1} 
Let $n\geq \max[N_0,\frac{1}{C^2m\log(2/\delta)}[1+\frac{4}{\varepsilon}\{p+2\log(2/\delta)\}\log(\frac{\lambda_{\max}(M_{F_n})}{b})]^2, (C')^2m\log(2/\delta)\{\frac{2L_n}{b}+\frac{1}{\lambda_{\min}(M_{F_n})}(C_1+C_2\frac{K_n}{b})\}^2]$ and assume that Conditions \ref{cond:boundedpsi}--\ref{cond:smoothness} hold. Then hen $A_T$ is $(\varepsilon,\delta)$--differentially private.
\end{theorem}

Theorem \ref{thm1} shows that our proposal leads to differentially private estimation. It builds on two lemmas, relegated to the Appendix, that  show that the smooth sensitivity of $T$ can indeed be upper bounded by twice its empirical gross error sensitivity. Note that the minimum sample size requirement depends on the values of $\{N_0,b,K_n,L_n,C_1,C_2\}$ defined in Conditions 1--3, as well as some constants $C$ and $C'$ resulting from our bounds on the error incurred by approximating the smooth sensitivity with the empirical gross-error-sensitivity.  We provide a discussion about the evaluation of these constants in the Appendix.

 \subsection{Examples}
 
Let us now present three important examples in order to illustrate how one can use readily available robust M-estimators and their influence functions to derive bounds on their empirical gross-error sensitivities. These quantities can in turn be used to release differentially private estimates $A_T(F_n)$ defined in \eqref{Gaussianmechanism1}. 

 \subsubsection*{Example 1: Location-scale model}

 We consider the location-scale model discussed in \cite[Chapter 6]{huberandronchetti2009}. Here we observe an iid random sample of univariate random variables $X_1,\dots,X_n$ with density function of the form $\frac{1}{\sigma}f(\frac{x-\mu}{\sigma})$, where $f$ is some known density function, $\mu$ is some unknown location parameter and $\sigma$ is an unknown positive scale parameter. The problem of simultaneous location and scale parameter estimation is motivated by invariance considerations. In particular, in order to make an M-estimate of location scale invariant, we must couple it with an estimate of scale.
If the underlying distribution $F$ is symmetric, location estimates $T$ and scale estimates $S$ typically are asymptotically independent, and the asymptotic behavior of $T$ depends on $S$ only through the asymptotic value $S(F)$. We can therefore afford to choose S on criteria other than low statistical variability. \cite{huber1964} generalized the maximum likelihood system of equations by considering simultaneous M-estimates of location and scale any pair of statistics $(T_n,S_n)$ determined by two equations of the form
\begin{equation*}
\sum_{i=1}^n\psi\Big(\frac{x_i-T_n}{S_n}\Big)=0 ~~~\mbox{ and }
\sum_{i=1}^n\chi\Big(\frac{x_i-T_n}{S_n}\Big)=0,
\end{equation*}
which lead $T_n=T(F_n)$ and $S_n=S(F_n)$ to be expressed in terms of functionals $T$ and $S$ defined by the population equations
\begin{equation*}
\int\psi\Big(\frac{x-T(F)}{S(F)}\Big)\mathrm{d}F(x)=0 ~~~\mbox{ and }
\int\chi\Big(\frac{x-T(F)}{S(F)}\Big)\mathrm{d}F(x)=0.
\end{equation*}
From the latter equations one can show that, if $\psi$ is odd and $\chi$ is even,  the influence functions of $T$ and $S$ are
\begin{equation}
\label{IFlocationscale}
\IF(x;T,F)=\frac{\psi\Big(\frac{x-T(F)}{S(F)}\Big)S(F)}{\int \psi'\Big(\frac{x-T(F)}{S(F)}\Big)\mathrm{d}F(x)}  ~~~\mbox{ and } ~~~ \IF(x;S,F)=\frac{\chi\Big(\frac{x-T(F)}{S(F)}\Big)S(F)}{\int \chi'\Big(\frac{x-T(F)}{S(F)}\Big)\frac{x-T(F)}{S(F)}\mathrm{d}F(x)}.
\end{equation}
The problem of robust joint estimation of location and scale was introduced in the seminal paper of \cite{huber1964}. In the important case of the normal model, where $F=\Phi$ is the standard normal distribution, a prominent example of the above system of equations is Huber's Proposal 2. In this case, $\psi(r)=\psi_c(r)=\min\{c,\max(-c,r)\}$ is the Huber function  and $\chi(r)=\chi_c(r)=\psi_c(r)^2-\kappa$, where $\kappa=\int \min(c^2,x^2)\mathrm{d}\Phi(x)$ is a constant that ensures Fisher consistency at the normal model i.e. $T(\Phi)=\mu$ and $S(\Phi)=\sigma^2$.  This particular choice of estimating equations and \eqref{IFlocationscale} show that the empirical gross-error sensitivities of $\hat{\mu}=T_n=T(F_n)$ and $\hat{\sigma}=S_n=S(F_n)$ are 
\begin{equation}
\label{GESlocationscale}
\gamma(T,F_n)=\frac{c \hat{\sigma}}{\frac{1}{n}\sum_{i=1}^nI_{\big|\frac{x_i-\hat{\mu}}{\hat\sigma}\big|<c} }~~~\mbox{ and }~~~\gamma(S,F_n)=\frac{(c^2-\kappa)\hat\sigma}{\frac{1}{n}\sum_{i=1}^n \Big(\frac{x_i-\hat\mu}{\hat\sigma}\Big)^2I_{\big|\frac{x_i-\hat{\mu}}{\hat\sigma}\big|<c} },
\end{equation}
where the last equation used that $\chi_c'(r)=\psi_c'(r)r$ almost everywhere and $I_E$ is the indicator function taking the value 1 under the event $E$ and is 0 otherwise. The formulas obtained in \eqref{GESlocationscale} can be used in the Gaussian mechanism  \eqref{Gaussianmechanism1} for obtaining private location and scale estimates. We  refer the reader to Chapter 6 in \cite{huberandronchetti2009} for more discussion and details on joint robust estimation of location and scale parameters.

 \subsubsection*{Example 2: Linear regression}
 
One can naturally build on the construction of the previous example to obtain robust estimators for the linear regression model
\begin{equation}
\label{LM}
y_i=x_i^T\beta+u_i, ~~~~~ \mbox{ for } i=1,\dots,n,
\end{equation}
where $y_i$ is the response variable,  $x_i\in\mathbb{R}^p$ the covariates and  the noise terms are $u_i\overset{iid}{\sim}N(0,\sigma^2).$  The estimator discussed 
here is a Mallows' type robust M-estimator defined as
\begin{equation}
\label{RLM}
(\hat{\beta},\hat{\sigma})=\argmin_{\beta,\sigma}\Big\{\sum_{i=1}^n\sigma\rho_c\Big(\frac{y_i+x_i^T\beta}{\sigma}\Big)w(x_i)+\kappa n\sigma\Big\},
\end{equation}
where $\rho_c$ is the Huber loss function with tuning parameter $c$, $\kappa=\int \min\{c^2,r^2\}\mathrm{d}\Phi(r)$ is a Fisher consistency constant for $\sigma$ and $w:\mathbb{R}^p\to\mathbb{R}_{\geq 0}$ is a downweighting function that controls the impact of outlying covariates on the estimators of  $\hat\beta$ and $\hat\sigma$  \citep{hampeletal1986}. This robust estimator uses Huber's Proposal 2 for the estimation of the scale parameter. In this case, the influence function of the estimator $\hat{\beta}=T(F_n)$ is
$$\IF(x,y;T,F)=M_F^{-1}\psi_c\Big(\frac{y-x^TT(F)}{S(F)}\Big)xw(x), $$
where $M_F=\int xx^Tw(x)\psi_c'(r)\mathrm{d}F$ and $r=\frac{y-x^TT(F)}{S(F)}$. Therefore $M_{F_n}=\frac{1}{n}\sum_{i=1}^nx_xx_i^Tw(x_i)I_{|\hat{r}_i|\leq c}$ with $\hat{r}_i=(y_i-x_i^T\hat\beta)/\hat\sigma$, and assuming that $\sup_x\|xw(x)\|\leq \tilde{K}$, we see that $\gamma(T,F_n)\leq \lambda_{\min}(M_{F_n})^{-1}c\tilde{K}$. This last bound can be used for the release of a differentially private estimates of $\beta$. Note also that using the derivations from Example 1 we also have that the empirical gross-error sensitivity of $\hat\sigma=S(F)$ is $\gamma(S,F_n)=[\frac{1}{n}\sum_{i=1}^n\hat{r}_i^2I_{|\hat{r}_i|\leq c}]^{-1}(c^2-\kappa)\hat\sigma$.

 \subsubsection*{Example 3: Generalized linear models}
 Generalized linear models \citep{mccullaghandnelder83} assume that  conditional on some covariates, the response variables belong to the exponential family i.e. the response variables $Y_1,\dots,Y_n$ are drawn independently from the densities of the form
$$ f(y_i;\theta_i)= \operatorname{exp}\Big[\big\{y_i\theta_i-b(\theta_i)\big\}/\phi+c(y_i,\phi) \Big], $$
where $a(\cdot)$, $b(\cdot)$ and $c(\cdot)$ are specific functions and $\phi$ a  nuisance parameter.
Thus $E(Y_i)=\mu_i=b'(\theta_i)$ and var($Y_i$)$=v(\mu_i)=\phi b''(\theta_i)$
and $g(\mu_i)=\eta_i=x_i^T \beta,$ where $\beta  \in  R^{p}$ is the vector of parameters, $x_i \in R^{p}$ is the set of explanatory variables and $g(\cdot)$ the link function. 

\cite{cantoniandronchetti01} proposed a class of M-estimators for GLM which can be viewed as a natural robustification of the quasilikelihood estimators of \cite{wedderburn74}. Their robust quasilikelihood  is 
\begin{equation*}
\label{robQL}
\rho_{n}(\beta)=\frac{1}{n}\sum\limits_{i=1}^nQ_M(y_i,x_i^T\beta),
\end{equation*}
where the functions $Q_M(y_i,x_i^T\beta)$ can be written as
\begin{equation*}
\label{robQLi}
Q_M(y_i,x_i^T\beta)=\int_{\tilde{s}}^{\mu_i}\nu(y_i,t)w(x_i)\mathrm{d}t-\frac{1}{n}\sum_{j=1}^{n}\int\limits_{\tilde{t}}^{\mu_j}E\big\{\nu(y_i,t)\big\}w(x_j)\mathrm{d}t
\end{equation*}
with $\nu(y_i,t)=\psi\{(y_i-t)/\sqrt{v(t)}\}/\sqrt{v(t)}$, $\tilde{s}$ such that $\psi\{(y_i-\tilde{s})/\sqrt{ v(\tilde{s})}\}=0$ and $\tilde{t}$ such that $E\big[ \psi\{(y_i-\tilde{s})/\sqrt{ v(\tilde{s})}\}\big]=0$.
The function $\psi(\cdot)$ is bounded and protects against large outliers in the responses, and $w(\cdot)$ downweights leverage points. The estimator of $\hat{\beta} $ of $\beta$ derived  from the minimization of this loss function  is the solution of the estimating equation 
\begin{equation}
\label{estGLM}
 \Psi^{(n)}(\beta)=\frac{1}{n}\sum\limits_{i=1}^n\Psi\big( y_i,x_i^T\beta \big)=\frac{1}{n}\sum\limits_{i=1}^n\bigg\{ \psi(r_i)\frac{1}{\sqrt{ v(\mu_i)}}w(x_i)\frac{\partial \mu_i}{\partial \beta}-a(\beta)\bigg\}=0,
\end{equation}
where $r_i=(y_i-\mu_i)/\sqrt{ v(\mu_i)}$ and $a(\beta)=n^{-1}\sum_{i=1}^n E\{ \psi(r_i)/\sqrt{ v(\mu_i)} \}w(x_i)\partial \mu_i/\partial \beta$ ensures Fisher consistency and can be computed using the formulas in Appendix A of \cite{cantoniandronchetti01}. We note that Appendix B of the same paper show that $M_{F}$ is of the form $\frac{1}{n}X^TBX$  and that these estimators and formulas are implemented in the function \texttt{glmrob} of the \texttt{R} package \texttt{robustbase}. They can be used to used to bound the empirical gross-error sensitivity with $\gamma(T,F_n)\leq \lambda_{\min}(M_{F_n})^{-1}K_n$ where $K_n$ is as in Condition 1 and will be depend on the choices of $\psi$ and $w$ as was the case in Example 2.

\subsection{Convergence rates}
We provide upper bounds for the convergence rates  of $A_T(F_n)$. Our result is an extension of Theorem 3 in \cite{chaudhuriandhsu2012}. 
\begin{theorem}
\label{rates}
Suppose Conditions \ref{cond:boundedpsi}--\ref{cond:hessian} hold. Then, for $\tau\in (0,1)$, with probability at least $1-\tau$
$$\|A_{T}(F_n)-T(F)\|\leq \|T(F_n)-T(F)\|+C\frac{\sqrt{\log(n)\log(2/\delta)}K_n\{\sqrt{p}+\sqrt{\log(1/\tau)}\}}{\varepsilon n} $$
for some  positive constant $C$. If in addition $\frac{K_n\sqrt{m\log(n)\log(1/\delta)}}{\sqrt{n}\varepsilon}\to 0$ as $n\to\infty$, then 
$$A_{T}(F_n)-T(F)=T(F_n)-T(F)+o_p(1/\sqrt{n}).$$
\end{theorem}

A direct consequence of the above result and \cite[Corollary 6.7]{huberandronchetti2009} is that $A_T(F_n)$ is asymptotically normally distributed as stated next.
\begin{corollary}
Assume that $p$ is fixed and that Conditions \ref{cond:boundedpsi}--\ref{cond:hessian} hold. Further assume that $\mathbb E_{F_{\theta_0}}[\|\Psi(X,\theta_0)\|^2]$ is nonzero and finite. If $\frac{K_n\sqrt{\log(n)\log(1/\delta)}}{\varepsilon\sqrt{n}}\to 0$ as $n\to\infty$, then we have that
$$\sqrt{n}(A_{T}(F_n)-T(F))\to_d N(0,V(T,F)) .$$
\end{corollary}
\begin{remark}
 This asymptotic normality result can be easily extended to the case where $p$ diverges as $n$ increases. In particular, invoking the results of \cite{heandshao2000} asymptotic normality holds assuming $\frac{p^2\log p}{n}\to 0$.  Note also that when $p$ diverges, $K_n$ will be diverging even for robust estimators as  componentwise boundedness of $\Psi$ implies that $K_n=O(\sqrt{p})$.
 \end{remark}

\subsection{Efficiency, truncation and robustness properties}
 \cite{smith2008,smith2011} introduced a class of asymptotically efficient point estimators obtained by averaging subsampled estimators and adding well calibrated noise using  the Laplace mechanism of \cite{dworketal2006}. 
 Unfortunately his construction relies heavily on the assumption that the diameter of the parameter space is known when calibrating the noise added to the output. Furthermore it is also assumed that we observe bounded random variables. Variants of this assumption are common in the differential privacy literature \citep{smith2011,lei2011,bassilyetal2014}. Our estimators can bypass these issues as long as the $\Psi$ diverges slower than $\sqrt{n}$. In particular, this is easily achievable with robust M-estimators since by construction they have a bounded $\Psi$. Alternatively, we  could  use truncated maximum likelihood score equations to obtain asymptotically efficient estimators  as shown next.
 \begin{corollary}
 \label{truncation}
Let $T_n$ denote the M-functional defined by the truncated score function $s_c(x,\theta)= \frac{\partial\log f_\theta(x)}{\partial\theta}w_{c,\theta}(x)$, where $w_{c,\theta}(x)=\min\{1,c/\|\frac{\partial\log f_\theta(x)}{\partial\theta}\|\}$, $c$ is some positive constant and  $f_{\theta_0}$ denotes the density of $F$. If $c\to\infty$ and $c\frac{\log(n)}{\sqrt{n}\varepsilon}\to 0$ as $n\to\infty$, then we have that 
$$\sqrt{n}(A_{T_n}(F_n)-\theta_0)\to_d N(0,I^{-1}(\theta_0)),$$
where $I(\theta_0)$ denotes the Fisher information matrix.
\end{corollary}
 The truncated maximum likelihood construction is reminiscent of the estimator of \cite{catoni2012}. The latter also uses a diverging level of truncation, but as a tool for achieving optimal non-asymptotic sub-Gaussian-type deviations for mean estimators under heavy tailed assumptions.

From a robust statistics point of view a diverging level of truncation is not a fully satisfactory solution. Indeed, it is well known that maximum likelihood estimators can be highly sensitive to the presence of small fractions of contamination in the data. This remains true for the truncated maximum likelihood estimator if the truncation level is allowed to diverge as it entails that the estimator will fail to have a bounded influence function asymptotically and will therefore not be robust in this sense. Interestingly, \cite{chaudhuriandhsu2012} showed that any differentially private algorithm needs to satisfy a somehow weaker degree of robustness. Our next Theorem provides a result in the same spirit for multivariate M-estimators.

\begin{theorem}
\label{lowerbound2}
Let $\varepsilon\in (0,\frac{\log 2}{2})$ and $\delta\in(0,\frac{\varepsilon}{17})$. Let $\mathfrak{F}$ be the family of all distributions over $\mathfrak{X}\subset\mathbb{R}^p$ and let $A$ be any $(\varepsilon,\delta)$-differentially private algorithm of $T(F)$. For all $n\in\mathbb{N}$ and $F\in\mathfrak{F}$ there exists a radius $\rho=\rho(n)=\frac{1}{n}\lceil\frac{\log 2}{2\varepsilon}\rceil$ and a distribution $G\in\mathfrak{F}$ with $d_{TV}(F,G)\leq \rho$, such that either 
$$\mathbb{E}_{F_n}\mathbb{E}_{A}\Big[\|A(\mathcal{D}(F_n))-T(F)\|\Big]\geq \frac{\rho}{16}\gamma(T,F) + o(\rho) $$
or
$$\mathbb{E}_{G_n}\mathbb{E}_{A}\Big[\|A(\mathcal{D}(G_n))-T(G)\|\Big]\geq \frac{\rho}{16}\gamma(T,F) + o(\rho), $$
where $F_n$ and $G_n$ denote empirical distributions obtained from $F$ and $G$ respectively.
\end{theorem}

Theorem \ref{lowerbound2} states that the convergence rates of any differentially private algorithm $A$ estimating the M-functional $T$ is lower bounded by  $\rho\gamma(T,F)$ in a small neighborhood of $F$. Therefore M-functionals $T$ with diverging influence functions will have slower convergence rates for any algorithm $A$ in all such neighborhoods. In this sense some degree of robustness is needed in order to obtain informative differential private algorithms and the theorem suggests that the influence function has to scale at most as $\rho^{-1}=O(\varepsilon n)$.

\section{Differentially private inference}

We now present our core results for privacy-preserving hypothesis testing building on the randomization scheme introduced in the previous section. 

\subsection{Background}
\label{backgroundtests}
We denote the partition of a $p$ dimensional vector $v$ into $p-k$ and $k$ components by $v=(v_{(1)}^T,v_{(2)}^T)^T$. We are interested in testing hypothesis of the form $H_0: \theta=\theta_0$, where $\theta_0=(\theta_{0(1)}^T,0^T)^T$ and $\theta_{0(1)}$ is unspecified against the alternative $H_1: \theta_{0(2)}\neq 0$ where $\theta_{0(1)}$ is unspecified. We assume throughout that the dimension $k$ is fixed. A well known result in statistics states that the Wald, score and likelihood ratio tests are asymptotically optimal and equivalent in the sense that they converge to the uniformly most powerful test \citep{lehmannandromano2006}. The level functionals of these test statistics can be approximated by functionals of the form
\begin{equation}
\label{levelfunctional}
\alpha(F_n):=1-H_k(nU(F_n)^TU(F_n))
\end{equation}
where $H_k(\cdot)$ is the cumulative distribution function of a $\chi_k^2$ random variable,  $U(F_n)$ is a standardized functional such that under the null hypothesis $U(F)=0$ and
\begin{equation}
\label{Ustand}
\sqrt{n}(U(F_n)-U(F))\to_d N(0,I_k).
\end{equation}
\cite{heritierandronchetti1994} proposed robust tests based on M-estimators. Their main advantage over their classical counterparts is they have bounded level and power influence functions. Therefore these tests are stable under small arbitrary contamination under both the null hypothesis and the alternative. Following \cite{heritierandronchetti1994} we therefore consider the three classes of tests described next.
\begin{enumerate}
\item A Wald-type test statistic is a quadratic statistic of the form
\begin{equation}
\label{Wald}
W(F_n):=T(F_n)_{(2)}^T(V(T,F)_{(22)})^{-1}T(F_n)_{(2)}. 
\end{equation}
\item A score (or Rao)-type test statistic has the form
\begin{equation}
\label{score}
R(F_n):=Z(T,F_n)^TU(T,F)^{-1}Z(T,F_n), 
\end{equation}
where $Z(T,F_n)=\frac{1}{n}\sum_{i=1}^n\Psi(X_i,T_R(F_n))_{(2)}$, $T_R$ is the restricted M-functional defined as the solution of
$$\int\Psi(x,T_R(F))_{(1)}\mathrm{d}F=0 ~~~\mbox{ and } ~~~ T_R(F)_{(2)}=0,$$
$U(T,F)=M_{(22.1)}V(T,F)_{(22)}M_{(22.1)}^T$ is a positive definite matrix and $M_{(22.1)}=M_{(22)}-M_{(21)}M_{(11)}^{-1}M_{(12)}$ with $M=M(T,F)$
\item A likelihood ratio-type test has the form
 \begin{equation}
 \label{LR}
 S(F_n):=\frac{2}{n}\sum_{i=1}^n\{\rho(x_i,T(F_n))-\rho(x_i,T_R(F_n))\}, 
 \end{equation}
where $\rho(x,0)=0$, $\frac{\partial}{\partial\theta}\rho(x,\theta)=\Psi(x,\theta)$ and $T$ and $T_R$ are the M-functionals of the full and restricted models respectively. As showed in \cite{heritierandronchetti1994} the likelihood ratio functional is asymptotically equivalent to the quadratic form $\tilde{S}(F):=U_{LR}(F)^TU_{LR}(F)$ where $U_{LR}(F)=M_{(22.1)}^{1/2}T(F)_{(2)}$.
\end{enumerate}
Note that in practice the matrices $M(T,F)$, $U(T,F)$ and $V(T,F)$ need to be estimated. We discuss this point in Section \ref{CVF}.

\subsection{Private inference based on the level gross-error sensitivity}
\label{PI}
We can use any of the robust test statistics described above to provide differential private p-values using an analogue construction to the one introduced for estimation in Section 3.
 Our proposal for differentially private testing is to build  p-values of the form 
$$A_{\alpha}(F_n):=\alpha(F_n)+ \gamma(\alpha;F_n)\frac{5\sqrt{2\log(n)\log(2/\delta)}}{\varepsilon n }Z,$$
 where $Z$ is an independent standard normal random variable. The rationale behind our construction is that $\gamma(\alpha,F_n)$ is the right scaling factor for applying the Gaussian mechanism to $\alpha(F_n)$ since it should roughly be of the same order as its smooth sensitivity. Note also that one can use $A_{\alpha}(F_n)$ to construct randomized counterparts to the test statistics  \eqref{Wald}, \eqref{score} and \eqref{LR} by simply computing
 $$Q(F_n):= H_k^{-1}(A_{\alpha}(F_n)),$$
 that is by evaluating the quantile function of a $\chi_k^
2$ at $A_{\alpha}(F_n)$. Note that we can also apply the Gaussian mechanism to the Wald, score and likelihood ratio type statistics of Section \ref{backgroundtests} and construct differentially private p-values from them. Indeed postprocessing preserves differential privacy so computing the induced p-values preserves the privacy guarantees \citep[Proposition 2.1]{dworkandroth2014}. Our theoretical results extend straightforwardly to this alternative approach and the numerical performance is nearly identical to the one presented in this paper in our experiments.
The following theorem establishes the differential privacy guarantee of our proposal.
\begin{theorem} 
\label{thm3}
Let $n\geq \max[N_0,\frac{1}{C^2m\log(2/\delta)}\{1+\frac{4}{\varepsilon}(p+2\log(2/\delta))\log(C_{n,k,U})\}^2, C_U^2 m\log(1/\delta)\frac{K_n^2}{\lambda_{\max}(M_{F_n})}\{1+\frac{2L_n}{b}+\frac{1}{\lambda_{\min}(M_{F_n})}(C_1+C_2\frac{K_n}{b})\}^2]$, where  $C_U$ and $C_{n,k,U}$ are constants depending on the test functional. If  Conditions \ref{cond:boundedpsi}--\ref{cond:smoothness} hold, then  $A_{\alpha}$ is $(\varepsilon,\delta)$-differentially private. 
\end{theorem}
The minimum sample size required in Theorem \ref{thm3} is similar to that of Theorem \ref{thm1}. In particular it also depends on  the same $\{N_0,b,K_n,L_n,C_1,C_2,C\}$, as well as the test specific constants $C_U$ and $C_{n,k,U}$ resulting from our bounds on the error incurred by approximating the smooth sensitivity of the level functionals by their the empirical gross-error sensitivity. A discussion on these constants can be found in the Appendix.

 \subsection{Examples}
 
The following two examples show how to upper bound empirical the level gross-error sensitivity $\gamma(\alpha,F_n)$ required for the construction of our differentially private p-values.  
 \subsubsection*{Example 4 : Testing and confidence intervals in linear regression}
 
 We consider the problem of hypothesis testing in the setting considered in Example 2. We focus on the same Mallow's estimator in combination with the Wald statistics $W_n=W(F_n)$ defined in \eqref{Wald} for hypothesis testing. We first note that from the chain rule,  the influence function of $W$ at the $F_n$ is
 $$\IF(x;W,F_n)= 2T(F_n)_{(2)}^T(V(T,F)_{(22)})^{-1}\IF(x;T,F_n)_{(2)}.$$
It follows that $\gamma(W,F_n)\leq 2\lambda_{\min}(V(T,F)_{(22)})^{-1}\|T(F_n)_{(2)}\|\gamma(T_{(2)},F_n)$ and the respective level gross-error sensitivity can be bounded as $\gamma(\alpha_W,F_n)\leq nH_k'(nW_n)\gamma(W,F_n).$ In the case of univariate null hypothesis of the form $H_0:\beta_j=0$ these expressions  become
$$\IF(x;W,F_n)= \frac{2T(F_n)_{j}}{V(T,F)_{jj}}\IF(x;T,F_n)_{j}~ \mbox{ and } ~ \gamma(\alpha_W,F_n)\leq 2nH_n'(nW_n) \frac{|T(F_n)_j|}{V(T,F)_{jj}} \|(M_{F_n}^{-1})_{j\cdot}\|K_n,$$
where $(M_{F_n}^{-1})_{j\cdot}$ denotes the $j$th row of $M_{F_n}^{-1}$. The above bound on $\gamma(\alpha_W,F_n)$ can be used in the Gaussian mechanism suggested in Section \ref{PI} for reporting differentially private p-values $A_{\alpha_W}(F_n)$ accompanying the regression slope estimates $A_T(F_n)$ of Example 2. 

We further note that since $(\varepsilon,\delta)$-differential privacy is not affected by post-processing, one can also construct confidence intervals using the reported p-value $A_{\alpha_W}(F_n)$. Since the asymptotic distribution of the Wald test is a $\chi^2_1$ for the null hypothesis $H_0:\beta_j=0$, a natural way to construct a confidence interval is to map the value $A_{\alpha_W}(F_n)$ to the quantile of $\chi^2_1$ and output the interval defined by its squared root. More precisely, one can first compute $Q_n^{(\varepsilon,\delta)}=H_1^{-1}(A_{\alpha_W}(F_n))$ and then report the differentially private confidence interval $(-\sqrt{Q_n^{(\varepsilon,\delta)}},\sqrt{Q_n^{(\varepsilon,\delta)}})$. 

 \subsubsection*{Example 5: Testing and confidence intervals in logistic regression}
 
Let us now return to the robust quasilikelihood estimator discussed in Example 3 and focus on the special case of binary regression with canonical link. Note that if one chooses $\psi(r)=r$ and $w(x)=1$ in \eqref{estGLM}, the resulting estimator is equivalent to logistic regression. In general  \eqref{estGLM} will take the form 
$$ \frac{1}{n}\sum_{i=1}^n\Big\{\psi(r_i)\frac{e^{\frac{1}{2}x_i^T\beta}}{1+e^{x_i^T\beta}}w(x_i)x_i-a(\beta)\Big\}=0,$$
where $r_i=(y_i-p_i)/\sqrt{p_i(1-p_i)}$ and  $p_i=\frac{e^{x_i^T\beta}}{1+e^{x_i^T\beta}}$. In this case, if $\sup_x\|xw(x)\|\leq \tilde{K}$ and $|\psi(r)|\leq c_\psi$, then the gross-error sensitivity of $\hat\beta=T(F_n)$ can be bounded as $\gamma(T,F_n)\leq 2\lambda_{\min}(M_{F_n})^{-1}c_\psi \tilde{K} $. For example if  we consider  the weight function  $w(x)=\{1,1/\|x\|\}$ and the Huber function $\psi(r)=\psi_c(r)$, then $\tilde{K}=1$ and $c_\psi$ is the constant of the Huber function. Note also that Appendix B in \cite{cantoniandronchetti01} provide formulas for $M_F$ when $\psi(r)=\psi_c(r)$ and this bound is readily obtained using standard functions in \texttt{R}. Furthermore the computation of the the gross-error sensitivity for the level functional of the Wald statistics follows from the same arguments discussed in Example 4. The extension of the proposed construction of confidence intervals is also immediate.

\subsection{Validity of the tests}
In this subsection we establish statistical consistency guarantees for our differentially private tests. The next theorem establishes rates of convergence and demonstrates the asymptotic equivalence between them and their non-private counterparts under both the null distribution and a local alternative.

\begin{theorem}
\label{ratestest}
Assume Conditions \ref{cond:boundedpsi} and \ref{cond:hessian} hold and let $\alpha(\cdot)$ be the level functional of any of the tests \eqref{Wald}--\eqref{LR}. Then, for $\tau\in(0,1)$, with probability at least $1-\tau$
$$|A_{\alpha}(F_n)-\alpha(F)|\leq |\alpha(F_n)-\alpha(F)|+C\frac{\sqrt{\log(n)\log(2/\delta)\log(2/\tau)} K_n}{\sqrt{n/k}\varepsilon} $$
for some positive constant $C$. Furthermore, if $\frac{K_n\sqrt{\log(n)\log(1/\delta)}}{\sqrt{n/k}\varepsilon}\to 0$ as $n\to\infty$ then
$$Q(F_n)=Q_0(F_n)+o_P(1),$$
where $Q_0(F_n)=H_k^{-1}(\alpha(F_n))$.
\end{theorem}
A direct consequence of Theorem \ref{ratestest} is that the asymptotic distribution of $Q(F_n)$ is the same as the one of its non-private counterpart $Q_0(F_n)$ computed from the level functional of any of the tests \eqref{Wald}--\eqref{LR}. Therefore the results of \cite{heritierandronchetti1994} also give the asymptotic distributions of $Q(F_n)$
 under both $H_0:\theta=\theta_0$ and $H_{1,n}:\theta=\theta_0+\frac{\Delta}{\sqrt{n}}$ for some $\Delta>0$. In particular, Propositions 1 and 2 of that paper establish that  \eqref{Wald} and \eqref{score} are asymptotically equivalent as they both converge to  $\chi_k^2$ under $H_0$ and to $\chi_k^2(\delta)$ with $\delta=\Delta^T V(T,F)^{-1}_{(22)}\Delta$ under $H_{1,n}$. Proposition 3 of the same paper shows that \eqref{LR} converges instead to a weighted sum of $k$ independent random variables distributed as $\chi_1^2$ under $H_0$ and to a weighted sum of $k$ independent random variables $\chi_1^2(\delta_i)$ for some $\delta_1,\dots,\delta_k> 0$ under $H_{1,n}$.

\subsection{Robustness properties of differentially private tests}
The tests associated with the differentially private p-values proposed in Section \ref{PI}  enjoy some degree of robustness by construction. In particular, it is not difficult to extend the lower bound of Theorem \ref{lowerbound2} to the level functionals  considered in this section.

\begin{theorem}
\label{lowerbound3}
Assume the conditions of Theorem \ref{lowerbound2}, but letting $A$ be any $(\varepsilon,\delta)$-differentially private algorithm of the level functional $\alpha(F)$ of either of the tests \eqref{Wald}--\eqref{LR}. Then either 
$$\mathbb{E}_{F_n}\mathbb{E}_{A}\Big[|A(\mathcal{D}(F_n))-\alpha(F)\|\Big]\geq \frac{\rho}{16}\bigg\lceil\frac{\log 2}{2\varepsilon}\bigg\rceil\mu\gamma(U,F)^2 + o\bigg(\rho\bigg\lceil\frac{\log 2}{2\varepsilon}\bigg\rceil\bigg) $$
or
$$\mathbb{E}_{G_n}\mathbb{E}_{A}\Big[|A(\mathcal{D}(G_n))-\alpha(G)|\Big]\geq \frac{\rho}{16}\bigg\lceil\frac{\log 2}{2\varepsilon}\bigg\rceil\mu\gamma(U,F)^2 + o\bigg(\rho\bigg\lceil\frac{\log 2}{2\varepsilon}\bigg\rceil\bigg), $$
where $\mu=-\frac{\partial}{\partial\zeta}H_k(q_{1-\alpha_0};\zeta)\Big|_{\zeta=0} $, $H_k(\cdot,\zeta)$ is the cumulative distribution function of a non-central $\chi^2_k(\zeta)$ with non-centrality parameter $\zeta\geq 0$, $q_{1-\alpha_0}$ is the $1-\alpha_0$ quantile of a $\chi^2_k$ distribution and $\alpha_0=\alpha(F)$ is the nominal level of the test.
\end{theorem}

Similar to Theorem \ref{lowerbound2} , Theorem \ref{lowerbound3} states that the convergence rates of any differentially private algorithm $A$ estimating the level functional $\alpha$ is lower bounded by the the gross-error sensitivity of $U(F)$ in a small neighborhood of $F$, where $U$  is defined in \eqref{levelfunctional} and \eqref{Ustand}. Therefore functionals $U$ with diverging influence functions will lead to  slower convergence rates for any algorithm $A$ in all such neighborhoods. The result suggests that the influence function has to scale at most as $\rho^{-1}=O(\varepsilon \sqrt{n})$. 

Note that the appearance of  the  quadratic term  $\gamma(U,F)^2$ in the lower bound is intuitive from the definition of $\alpha(F)$ and is in line with the robustness characterization of the level influence function of \citep{heritierandronchetti1994,ronchettiandtrojani2001}. In fact we can extend the robustness results of these papers to our setting and show that our tests have stable level and power functions in shrinking contamination neighborhoods of the model when $\Psi$ is bounded. 

We need to introduce additional notation in order to state the result.
Consider the $(t,n)$-contamination neighborhoods of $F_{\theta_0}$ defined by
$$\mathfrak{U}_{t,n}(F_{\theta_0}):=\bigg\{F_{t,n,G}^0=\bigg(1-\frac{t}{\sqrt{n}}\bigg)F_{\theta_0}+\frac{t}{\sqrt{n}}G,~G \mbox{ arbitrary} \bigg\} $$
and let $U_n=U(F_n)$ be a statistical functional with  bounded influence function and such that $U(F)=0$ and 
$$\sqrt{n}(U(F_n)-U(F_{t,n,G}))\to_d N(0,I_k)   $$
uniformly over the sequence of $(t,n)$-neighborhoods $\mathfrak{U}_{t,n}(F_{\theta_0})$. Further let 
$$\{F_{\eta,n}^{alt}\}_{n\in\mathbb{N}}:=\bigg\{ \bigg(1-\frac{\eta}{\sqrt{n}}\bigg)F_{\theta_0}+\frac{\eta}{\sqrt{n}}F_{\theta_1}\bigg\}_{n\in\mathbb{N}}$$
be a sequence of local alternatives to $F_{\theta_0}$ and 
$$\mathfrak{U}_{t,n}(F_{\eta,,n}^{alt}):=\bigg\{F_{t,n,G}^1:=\bigg(1-\frac{t}{\sqrt{n}}\bigg)F_{\eta,n}^{alt}+\frac{t}{\sqrt{n}}G,~G \mbox{ arbitrary} \bigg\}  $$
be the corresponding neighborhood of $F_{\theta,n}^{alt}$ for a given $n$.  We denote by $\{F_{t,n,G}^0\}_{n\in\mathbb{N}}$ a sequence of $(t,n,G)$-contaminations of the underlying null distribution $F_{\theta_0}$, each of them belonging to the neighborhood $\mathfrak{U}_{t,n}(F_{\theta_0})$. Similarly, we denote by $\{F_{t,n,G}^1\}_{n\in\mathbb{N}}$ a sequence of $(t,n,G)$-contaminations of the underlying local alternatives $F_{\eta,n}^{alt}$, each of them belonging to the neighborhood $\mathfrak{U}_{t,n}(F_{\eta,n}^{alt})$. Finally, we denote by $A_\beta$ and $\beta$ the power functionals of the tests based on $A_\alpha$ and $\alpha$ respectively.

The following corollary follows  from \cite[Theorems 1--3]{ronchettiandtrojani2001} and  Theorem \ref{ratestest}. It shows that the level and power of our differentially private tests are  stable in the contamination neighborhoods $\mathfrak{U}_{t,n}(F_{\theta_0})$ and $\mathfrak{U}_{t,n}(F_{\eta,,n}^{alt})$ when the influence function of the functional $U$ is bounded. 
\begin{corollary}
\label{IFtests}
Our differentially private Wald, score and likelihood ratio type tests  have stable level and power functionals when $K_n<\infty$ in the sense that for all $G$
\begin{equation*}
\begin{aligned}
\label{levelexpansion}
\lim_{n\to \infty}A_\alpha(F_{t,n,G})=&\lim_{n\to \infty}\alpha(F_{t,n,G})\\
 =&\alpha_0+t^2\mu \Big\|\int \IF(x;U,F_{\theta_0})\mathrm{d}G(x)\Big\|^2+o(t^2) 
\end{aligned}
\end{equation*}
and 
\begin{equation*}
\begin{aligned}
\label{powerexpansion}
&\lim_{n\to \infty}A_\beta(F_{t,n,G}^1)=\lim_{n\to \infty}\beta(F_{t,n,G}^1) \\
& =\lim_{n\to\infty}\beta(F_{\eta,n}^{alt})  \\&+2\mu t \eta \int \IF(x;U,F_{\eta,n}^{alt})^T\mathrm{d}G(x) \int\IF(x;U,F_{\theta_0})\mathrm{d}F_{\theta_1}(x)+o(\eta), 
\end{aligned}
\end{equation*}
 where $\mu$ is as in Theorem \ref{lowerbound3}. 
\end{corollary}

\subsection{Accounting for the change of variance sensitivity}
\label{CVF}

In practice the standardizing matrices $M(T,F)$, $U(T,F)$ and $V(T,F)$ are estimated, so the actual form of the functional $U$ defining the test functional is 
$$ U(F_n)=S(F_n)^{-1/2}\tilde{T}(F_n),$$
where $\tilde{T}$ is such that $\tilde{T}(F)=0$ and $\sqrt{n}(\tilde{T}(F_n)-\tilde{T}(F))\to_d N(0,S(F))$. The general construction of Section \ref{PI} is still valid provided additional regularity conditions on $\Psi$ hold. In particular, it remains true that $\gamma(\alpha,F)$ can be used to upper bound $\tilde{\Gamma}_n$ provided $\Big|\frac{\partial}{\partial\theta_j}\dot{\Psi}\Big|<\infty$ for all $j=1,\dots,p$. This condition implies third order infinitesimal robustness in the sense of \cite{lavecchiaetal2012}.   From a practical point of view an upper bound on $\gamma(\alpha,F_n)$ can be computed in this case using both the influence function and the change of variance function of $T$. The latter accounts for the fact the $S(F)$ is also estimated. We refer the reader to the Appendix for the precise form of the the change of variance function of general M-estimators and a more detailed discussion of the implications of estimating the variance in the noise calibration of our Gaussian mechanism.

\section{Numerical examples}
\label{simulations}

We investigate the finite sample performance of our proposals with simulated and real data.  We  focus on a linear regression setting where we obtain consistent slope parameter estimates at the model and show that our differentially private tests reach the desired nominal level and has power under the alternative even in mildly contaminated scenarios.  We first present a simulation
 experiments that shows the statistical performance of our methods in  small samples before turning to a real data example with a large sample size. For the sake of space we relegate to the Appendix  a more extended discussion about other existing methods, some complementary simulation results and  a discussion of the evaluation of the constants of Theorems 1 and 4.

\subsection{Synthetic data}

We consider a simulation setting  similar to  the one of \cite{salibianbarreraetal2016} in order to explore the behavior of our consistent differentially private estimates and  illustrate the efficiency loss incurred by them, relative to their non private counterparts. We generate the linear regression model \eqref{LM} with $\beta=(1,1,0,0)^T$, $x_i\sim N(0,V)$ and $V=\{0.5^{|j-k|}\}_{j,k=1}^4$. We illustrate the effect of small amounts of contaminated data by generating outliers in the responses as well as bad leverage points. This was done by replacing $1\%$ of the values of $y$ and $x_2$ with observations following a $N(12,0.1^2)$ and a $N(5,0.1^2)$ distribution respectively. All the results reported below were obtained over $5000$ replications and sample sizes ranging from $n=100$ to $n=1000$.

 The differentially private estimates considered here is the same Mallow's type robust regression estimator of Example 2. In particular, we consider the robust estimators of $\beta$ defined by
\begin{equation*}
\label{RLM2}
(\hat{\beta}_0,\hat{\beta},\hat{\sigma})=\argmin_{\beta,\sigma}\Big\{\sum_{i=1}^n\sigma\rho_c\Big(\frac{y_i-\beta_0+x_i^T\beta}{\sigma}\Big)w(x_i)+\kappa_cn\sigma\Big\},
\end{equation*}
where $\rho_c$ is the Huber loss function with tuning parameter $c$, $w:\mathbb{R}^p\to\mathbb{R}_{\geq 0}$ is a downweighting function and $\kappa_c=\int \min\{x^2,c^2\}\mathrm{d}\Phi(x)$ is a constant ensuring that $\hat\sigma$ is consistent. In all our simulations we set $c=1.345$ and $w(x)=\min\{1,2/\|x\|_2\}$. 
This robust estimator uses Huber's Proposal 2 for the estimation of the scale parameter \citep{huberandronchetti2009}.  We computed it using the function \texttt{rlm} of the $\texttt{R}$ package ``MASS".
\begin{figure}[h]
\begin{center}
    \includegraphics[width = 5.8in, height=4in]{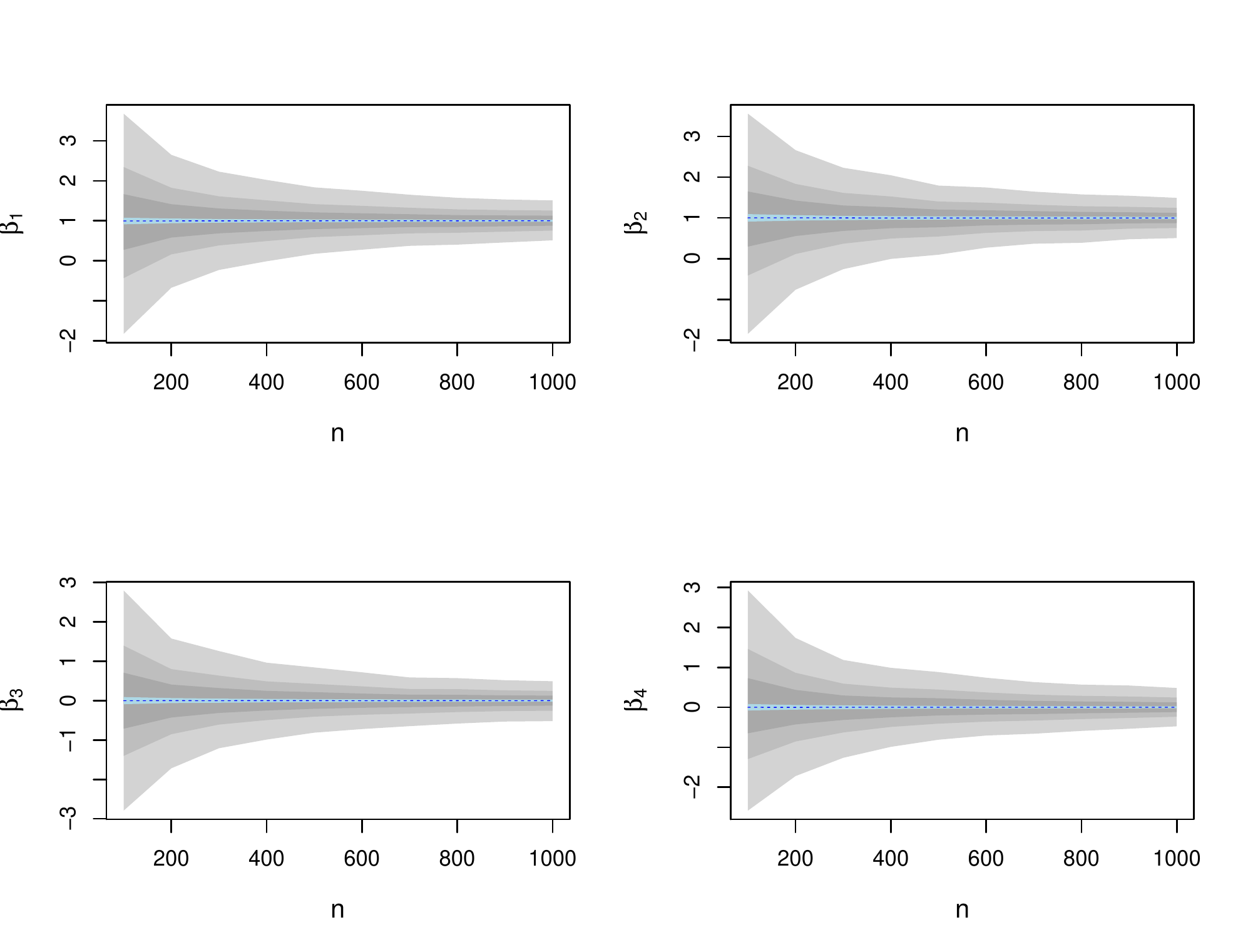}
\end{center}
\vspace{-0.5cm}
\caption{ {\small The plots show the componentwise estimation error of the
 parameter $\beta$ for clean data sets ranging from size $n=100$ to $n=1000$. The dotted dark blue line shows the median estimated value of the target robust estimator while the light blue shaded area give pointwise quartiles of the same estimator. The larger shaded gray areas give the pointwise quartiles of the estimated differentially private estimators with privacy parameters $\varepsilon=\{0.2,0.1,0.05\}$.} }
\label{estimation}
\end{figure}
Figure \ref{estimation} shows how the level of privacy affects the performance of estimation relative to that of the target robust estimator. In particular, it illustrates the slower convergence of our differential private estimators for the range of privacy parameters  $\varepsilon=\{0.2,0.1,0.05\}$ and $\delta=1/n^2$. Figure \ref{accuracytest} shows the empirical level of the Wald statistics for testing the null hypothesis $H_0: \beta_3=\beta_4=0$ with increasing sample sizes and nominal level of $5\%$. We see that all the tests have good empirical coverage and that as expected the differentially private tests are not too sensitive to the presence of a small amount of contamination. 
Interestingly, the empirical levels of the robust test and the differentially private one are nearly identical when the privacy parameters  $\varepsilon=\{1,0.1\}$ and $n\geq 200$. When we choose the very stringent $\varepsilon=0.001$ the noise added to the target p-value is so large that the resulting test amounts to flipping a coin.

In order to explore the power of our tests we set the regression parameter $\beta$ to $(1,1,\nu,0)^T$, where $\nu$ varied in the range $[-0.5,0.5]$. As seen in Figure \ref{powerstability} (a) the power function of the three tests considered is almost indistinguishable when the data follows the normal model \eqref{LM}. Figure \ref{powerstability} (b) shows that the power functions of the robust Wald tests and the derived differentially private test remain almost  identical to the one they have without contamination. This reflects the power function stability result established in Theorem \ref{powerexpansion}. From the same figure,  we clearly see that  the power function of the Wald test constructed using least squares estimator is shifted as a result of a small amount of contamination. 
\begin{figure}[h!]
\begin{center}
    \includegraphics[width = 5in, height=2.7in]{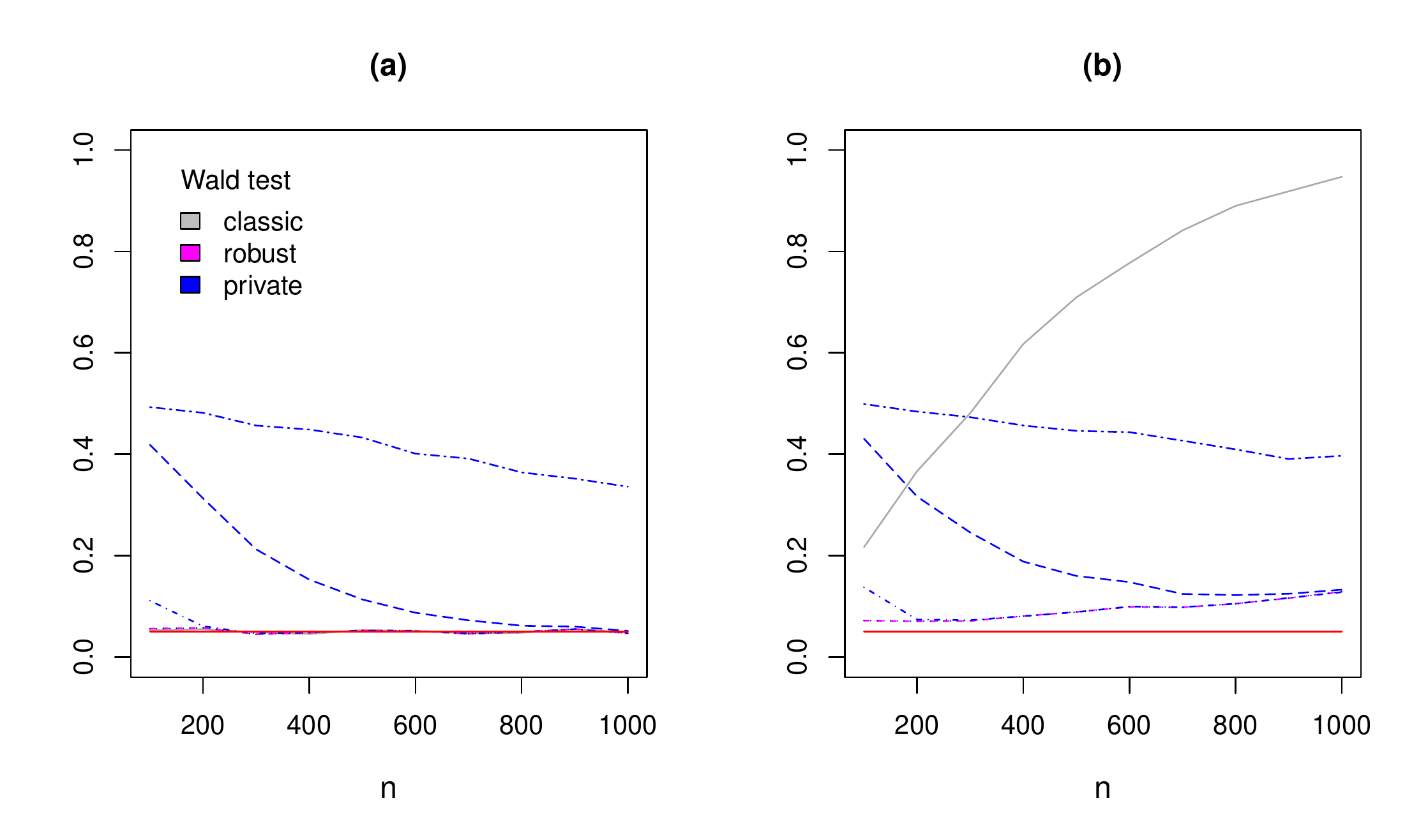}
\end{center}
\vspace{-0.5cm}
\caption{ {\small (a) shows the convergence of our Wald statistic to the nominal level $0.05$  at the model while (b) shows its behavior under $1\%$ contamination. We report four empirical differentially private level curves: dotted lines, $\varepsilon=1$; dash-dotted lines, $\varepsilon=0.1$; dash-dotted lines, $\varepsilon=0.01$, two-dashed lines, $\varepsilon=0.001$.}}
\label{accuracytest}
\end{figure}
\begin{figure}[h!]
\begin{center}
    \includegraphics[width = 5in, height=2.7in]{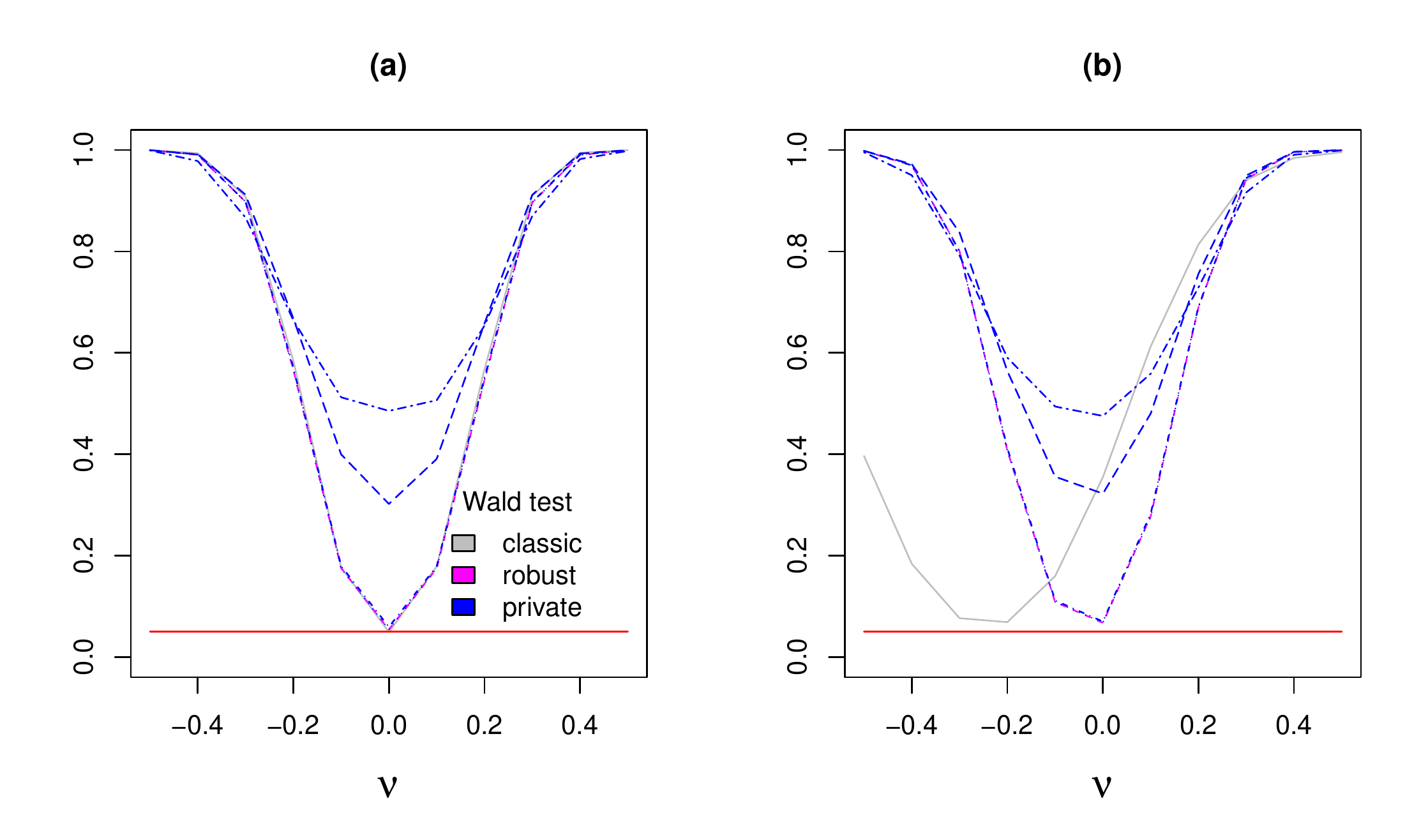}
\end{center}
\vspace{-0.5cm}
\caption{ {\small (a) shows the power function of our Wald statistic at the model when $n=200$ and $\beta_3\in[-0.5,0.5]$; (b) shows its behavior under $1\%$ contamination.  We report four empirical differentially private power curves: dotted lines, $\varepsilon=1$; dash-dotted lines, $\varepsilon=0.1$; dash-dotted lines, $\varepsilon=0.01$, two-dashed lines, $\varepsilon=0.001$.}}
\label{powerstability}
\end{figure}

\newpage
\subsection{Application to housing price data}

We revisit the housing price data set considered in \cite{lei2011}. The data consist of $348'189$ houses sold in the San Francisco Bay Area Between $2003$
and $2006$, for which we have the price, size, year of transaction, and county in which the house is located. The data set has two continuous covariates (price and size), one ordinal variable with 4 levels (year), and one categorical variable (county) with 9 levels. We exclude the observations with missing entries and follow the preprocessing suggested in \cite{lei2011}, i.e. we filter out data points with price outside the range of $\$10^5\sim\$9\times 10^5 $ or with size larger than $3'000$ squared feet. After preprocessing, we have $250'070$ observations and the county variables has $6$ levels after combination.   We also consider the same data without filtering price and size, in which case we are left with $286'537$ observations. 
We fitted a simple linear regression model in order to predict the housing price using ordinary least squares, a robust estimator and differentially private estimators. We computed the private estimator described in 5.1  as well as the differentially private M-estimators based on a perturbed histogram with enhanced thresholding as in \cite{lei2011}. We assess the performance of the differentially private regression coefficients by comparing them with their non-private counterparts. More specifically, we look at the componentwise relative deviance from the non-private estimates $d_j=|\hat{\beta}^{DP}_j/\hat{\beta}_j-1|$ where $\hat{\beta}_j$ stands for the $j$th regression coefficient of either the ordinary least squares or the robust estimator, and $\hat{\beta}^{DP}_j$ is its differentially private counterpart. In order to account for the randomness of the Gaussian mechanism, we report the mean square error of the deviations $d_j$ obtained over $500$ realizations. The results are summarized in Tables \ref{table1} and \ref{table2}.
\begin{table}[h]
\begin{center}
\caption{ {\small Linear regression coefficients using the Bay housing data after preprocessing. The second and third columns give the regression coefficients obtained by ordinary least squares and the robust Mallow's estimator without privacy guarantees. We compare the performance of their differentially private counterparts using the perturbed histogram approach and our Gaussian mechanism for a fixed privacy level $\varepsilon=0.1$. The reported number is the componentwise root mean square relative error over 1000 realizations}.\label{table1}}
\begin{tabular}{c c c |c c c  } 
       &   & & \multicolumn{3}{c}{\bf $\varepsilon=0.1$  }    \\

\hline
 {\bf Method} & {\bf OLS}  & {\bf Rob}& {\bf PH$_{OLS}$ }& {\bf PH$_{Rob}$ }&   {\bf DP}  \\ 
\hline
Intercept &135141  &  118479     & 8.9 &  10.4  & 1.4$\times 10^{-4}$     \\ 
Size      &  209   &  216        & 4.0  &  5.1   & 7.3$\times 10^{-2}$     \\
Year      & 56375  &  58136      & 2.6  &  5.2   & 2.8$\times 10^{-4}$     \\ 
County 2  &-53765  &  -59605     & 8.1  &  7.6   & 2.9$\times 10^{-4}$     \\
County 3  &146593  &  149202     & 2.7  &  3.8   & 1.1$\times 10^{-4}$     \\ 
County 4  &-27546  &  -29681     & 37.7 &  28.4  & 5.2$\times 10^{-4}$     \\
County 5  & 45828  &  41184      & 7.8  &  16.5  & 4.1$\times 10^{-4}$     \\ 
County 6  &-140738 &  -139780    & 3.6  &  7.7   & 1.1$\times 10^{-4}$     \\
\hline
\end{tabular}
\end{center}
\end{table}
\begin{table}[h]
\begin{center}
\caption{ \small{Linear regression coefficients using the raw Bay housing data without preprocessing. The reported numbers are as in Table \ref{table1}}. \label{table2}}
\begin{tabular}{c c c |c c c} 
       &   & & \multicolumn{3}{c}{\bf $\varepsilon=0.1$  }    \\

\hline
 {\bf Method} & {\bf OLS}  & {\bf Rob}& {\bf PH$_{OLS}$ }& {\bf PH$_{Rob}$ }&   {\bf DP} \\ 
\hline
Intercept       &  456344  &  101524  & 33.4  & 28.6  &  1.5$\times 10^{-4}$     \\ 
Size            &  0.5     &  229     & 247.1 & 229.3 &  6.2$\times 10^{-2}$     \\
Year            &  71241   &  65170   &  87.8 & 85.7  &  2.2$\times 10^{-4}$    \\ 
County 2        &  -11261  &  -53727  & 416.8 & 376.9 &  2.9$\times 10^{-4}$    \\
County 3        &  275058  &  196967  & 82.4  & 80.7  &  7.5$\times 10^{-5}$    \\ 
County 4        &  -16425  &  -29337  & 569.0 & 519.1 &  4.8$\times 10^{-4}$    \\
County 5        &  98775   &  57524   & 101.9 & 95.9  &  2.6$\times 10^{-4}$    \\ 
County 6        &  -149027 &  -152499 & 143.3 & 141.2 &  9.2$\times 10^{-5}$    \\
\hline
\end{tabular}
\end{center}
\end{table}

It is interesting to notice that with the preprocessed data the least squares fit and the robust fit are very similar. However with the raw data, the large unfiltered values of price and size affect to a greater extent the estimator of \cite{lei2011}. The accuracy of this estimator also deteriorates for the raw data as reflected by the larger mean squared deviations obtained in this case. On the other hand, our differentially private estimators give similar results for both preprocessed and raw data, in terms of values of the fitted regression coefficients and mean squared deviations from the target robust estimates. This is a particularly  desirable feature when privacy is an issue since researchers are likely to have limited access to the data and hence carrying out a careful preprocessing might not be possible. Note also that for the same level of privacy $\varepsilon=10^{-1}$, our method provides much more accurate estimation. The poorer performance of the histogram estimator is to be expected as it suffers from the curse of dimensionality. In this particular example Lei's estimator  effectively reduces the sample size to only 2400 pseudo observations that can be sampled from the differentially private estimated histogram.

We see from the reported values in Tables \ref{table1}--\ref{table2} that  the accuracy of our private estimator is comparable with that of the perturbed histogram if we impose the much stronger privacy requirement  $\varepsilon=10^{-3}$. This feature is also very appealing in practice and confirms what our theory predicts and what we observed in simulations: we can afford a fixed privacy budget with a smaller sample size or equivalently, for a  fixed sample size we can ensure a higher level of privacy using our methods. Note that given the large sample size of this data set, unsurprisingly all the covariates are significantly predictive for the non-private estimators. All univariate Wald statistics for the slope parameters in this example yield p-values smaller than $10^{-16}$ for the non-private estimators. Since our differentially private p-values give similar results we chose not to report them. 

\section{Concluding remarks}
 
 We introduced a general framework for differentially private statistical inference for parametric models based on M-estimators. The central idea of our approach is to leverage tools from robust statistics in the design of a mechanism for the release of differentially private statistical outputs. In particular, we release noisy versions of statistics of interest that we view as functionals of the empirical distribution induced by the data. We use a bound of their influence function in order to scale the random perturbation added to the desired statistics to guarantee privacy. As a result, we propose a new class of consistent differentially private estimators that can be easily and efficiently computed, and provide  a general framework for parametric hypothesis testing with privacy guarantees. 
  
 An interesting extension to be explored in the future is the construction of differentially private tests in the context of nonparametric and high-dimensional regression. In principle the idea of using the influence function to calibrate the noise added to test functionals also seems intuitive in these settings, but the technical challenge of these extensions is twofold. First, there  are no general results regarding the level influence function of tests for these settings.   Second, the influence function of nonparametric and high-dimensional penalized estimators has been formulated for a fixed tuning parameter \citep{christmannandsteinwart2007, marco2017}. Since in practice this parameter is usually chosen by some data driven criterion, it would be necessary to account for this selection step in the derivation of differentially private statistics following the approach of this work.    Another interesting direction for future research is to explore whether information-standardized influence functions could be used to derive better or more general differentially private estimators \citep{hampeletal1986, heandsimpson1992}. It would also be interesting to explore the construction of tests based on alternative approaches to differential privacy such as objective function perturbation \citep{chaudhuriandmonteleoni2008,chaudhurietal2011,kieferetal2012} or stochastic gradient descent \citep{rajkumarandargawal2012,bassilyetal2014,wangetal2015}.

\titleformat{\section}{\Large\bfseries}{\appendixname~\thesection :}{0.5em}{}

\newpage
 \pagenumbering{arabic}
    \setcounter{page}{1}
\title{\textsc{Supplementary File for ``Privacy-preserving parametric inference: a case for robust statistics''}  }
\vskip 1.5cm
\author{Marco Avella-Medina$^{*}$ }

\maketitle
\begin{appendices}

\section{proof of main results}

\subsection*{Proof of Theorem \ref{thm1}}

\begin{proof}
Our argument consists of using Lemmas  \ref{SS} and \ref{GES} to show that $\frac{\sqrt{\log(n)}}{n}\gamma(T,F_n)$ upper bounds the $\xi$-smooth sensitivity of the M-functional $T$. This suffices to show the desired result since choosing $\xi= \frac{\varepsilon}{4\{p+2\log(2/\delta)\}}$  guarantees $(\varepsilon,\delta)$-differential privacy as shown in \cite[Lemmas 2.6 and 2.9]{nissimetal2007}. 

From Lemma 2 we have that $\sqrt{\log n}\gamma(T,F_n)>\Gamma_n$ for $n\geq (C')^2 m\log(1/\delta)\{\frac{2L_n}{b\lambda_{\min}(M_{F_n})}(C_1+C_2K_n/b)\}^2$. Given Lemma 1, it therefore remains to show that 
\begin{equation}
\label{thm1.1}
\frac{\sqrt{\log n}\gamma(T,F_n)}{n} \geq \frac{1}{bn}K_n\exp\Big(-\xi C\sqrt{mn\log(2/\delta)}+\xi\Big).
\end{equation}
Further note that  $\gamma(T,F_n)\geq K_n/B_n$, where $B_n=\lambda_{\max}(M_F)$. Hence in order to show \eqref{thm1.1} it would suffice to establish that
$$\sqrt{\log(n)} \geq \frac{B_n}{b}\exp\Big(-\xi C\sqrt{mn\log(2/\delta)}+\xi\Big) $$
or equivalently 
\begin{equation}
\label{thm1.2}
2C\xi\sqrt{mn\log(2/\delta)}-2\xi-2\log(B_n/b) \geq -\log\log(n).
\end{equation}
Since  $\xi\leq \frac{\varepsilon}{4\{p+2\log(2/\delta)\}}\leq 1$, the left hand side of \eqref{thm1.2} will be nonnegative if   
\begin{equation*}
\label{thm1.3}
n\geq \Big\{ 1+\frac{\log(B_n/b)}{\xi}\Big\}^2\frac{1}{C^2m\log(2/\delta)}\geq \Big[ 1+\frac{4\{p+2\log(2/\delta)\}\log(B_n/b)}{\varepsilon}\Big]^2\frac{1}{C^2m\log(2/\delta)}
\end{equation*}
which holds by assumption.  We have thus established  \eqref{thm1.1} and hence that $\frac{\sqrt{\log(n)}}{n}\gamma(T,F_n)$ upper bounds the smooth sensitivity of $T$. Therefore the Gaussian mechanism with scaling $ \gamma(T,F_n)\frac{5\sqrt{2\log(n)\log(2/\delta)}}{\varepsilon n}$ guarantees $(\varepsilon, \delta)$-differential privacy \cite[Lemmas 2.6 and 2.9]{nissimetal2007}. 
\end{proof}

In addition to Conditions \ref{cond:boundedpsi}--\ref{cond:smoothness} discussed in the main document, the statements of Lemmas  \ref{SS} and \ref{GES} require three additional definitions introduced in \cite{chaudhuriandhsu2012}. The first two are fixed scale versions  of the influence function and the gross error sensitivity, i.e. for a fixed  $\rho>0$, we define
$$\IF_{\rho}(x;T,F):=\frac{T((1-\rho)F+\rho\delta_x)-T(F)}{\rho} $$
and
$$\gamma_{\rho}(T,F):=\sup_{x\in\mathfrak{X}}\|\IF_{\rho}(x;T,F)\| $$
The third  important quantity  appearing in our analysis is the supremum, over a Borel-Cantelli type neighborhood, of  the gross-error sensitivity  i.e.
\begin{equation}
\label{Gamma}
\Gamma_n:=\sup\bigg\{ \gamma_{1/n}(T,G) : d_{\infty}(F_n,G)\leq C\sqrt{\frac{m\log(2/\delta)}{n}} \bigg\}. 
\end{equation}
We are now ready to state the two main auxiliary lemmas.
 \begin{lemma}
\label{SS}
Assume Conditions \ref{cond:boundedpsi} and \ref{cond:hessian} hold. Then
$$\SS_{\xi}(T,\mathcal{D}(F_n))\leq \max\bigg\{\frac{2\Gamma_n}{n}, \frac{1}{bn}K_n\exp\Big(-C\xi \sqrt{mn\log(2/\delta)}+\xi\Big)\bigg\},$$
where $C$ is as in \eqref{Gamma}.
\end{lemma}
\begin{proof}
We adapt Lemma 1 in \cite{chaudhuriandhsu2012} to our setting.
We will show that for any $\mathcal{D}(G_1)\in \RR^{n\times m}$ we have that
$$e^{-\xi d_H(\mathcal{D}(F_n),\mathcal{D}(G_1))}\LS\big(T,\mathcal{D}(G_1)\big)\leq\max\Big\{2\Gamma_n/n, \frac{1}{bn}K_n\exp\big(-\xi nr_n+\xi \big)\Big\},$$
where $r_n= C\sqrt{\frac{m\log(2/\delta)}{n}}$. For this we consider two possible cases. First suppose that $[d_H\big(\mathcal{D}(F_n),\mathcal{D}(G_1)\big)+1]/n > r_n$. Letting $G_1'$ such that $d_H\big(\mathcal{D}(G_1),\mathcal{D}(G_1')\big)=1$ and taking $\rho=1$ in Lemma \ref{intermediatepoint} we get that $\LS\big(T,\mathcal{D}(G_1)\big)\leq \frac{K_n}{bn}$ since
\begin{align}
\label{lem1.1}
\|T(G_1)-T(G_1')\|\leq & \Big\|\int_0^1\int\IF(x;T,(1-t)G_1+tG_1')\mathrm{d}(G_1-G_1')\mathrm{d}t\Big\| \nonumber \\
\leq &d_\infty(G_1,G_1') \sup_{t\in[0,1]}\gamma(T,(1-t)G_1+tG_1')\nonumber \\
\leq &\frac{K_n}{bn}.
\end{align}
Therefore 
$$e^{-\xi d_H(\mathcal{D}(F_n),\mathcal{D}(G_1))}\LS\big(T,\mathcal{D}(G_1)\big)\leq \frac{K_n}{bn} \exp\big(-\xi nr_n+\xi\big).$$
Suppose now that $[d_H\big(\mathcal{D}(F_n),\mathcal{D}(G_1)\big)+1]/n \leq r_n$ and fix $\mathcal{D}(G_2)\in\RR^{n\times m}$ such that $d_H\big(\mathcal{D}(G_1),\mathcal{D}(G_2)\big)=1$. Let $j\in\{1,\dots,n\}$ be the index at which $\mathcal{D}(G_1)$ and $\mathcal{D}(G_2)$ differ.  Finally let $\mathcal{D}(G_3)\in\RR^{(n-1)\times m}$ be the data set obtained by removing the $j$th element of $\mathcal{D}(G_1)$. Then by the triangle inequality
$$d_{\infty}(F_n,G_3)\leq d_{\infty}(F_n,G_1)+d_{\infty}(G_3,G_1)\leq \big[d_H\big(\mathcal{D}(F_n),\mathcal{D}(G_1)\big)+1\big]/n\leq r_n $$
and hence $\gamma_{1/n}(T,G_3)\leq \Gamma_n$. Furthermore, using the triangle inequality we have that
\begin{align*}
\|T(G_1)-T(G_2)\|&=\|T(G_1)-T(G_3)+T(G_3)-T(G_2)\| \nonumber\\
 &=\frac{1}{n}\|\IF_{1/n}(x_j;T,G_3)-\IF_{1/n}(x_j';T,G_3)\|\nonumber\\
 &\leq \frac{2}{n}\gamma_{1/n}(T,G_3)\nonumber\\
 & \leq \frac{2\Gamma_n}{n}. \label{lem1.2}
\end{align*}
Since the last bound holds for any choice of $\mathcal{D}(G_2)$ we see that $\LS\big(T,\mathcal{D}(G_1)\big)\leq 2\Gamma_n/n$ and consequently $e^{-\xi d_H(\mathcal{D}(F_n),\mathcal{D}(G_1))}\LS\big(T,\mathcal{D}(G_1)\big)\leq 2\Gamma_n/n$.
\end{proof}
\begin{lemma}
\label{GES}
Assume  Conditions \ref{cond:boundedpsi}--\ref{cond:smoothness} hold. Then 
$$\Gamma_n\leq 2\gamma(T,F_n) +
C'\sqrt{\frac{m\log(2/\delta)}{n}}K_n\lambda_{\max}(M_{F_n}^{-1})\Big\{2L_n/b+\lambda_{\max}(M_{F_n}^{-1})(C_1+2C_2K_n/b)\Big\}$$
for some positive constant $C'$.
\end{lemma}
\begin{proof}
First note that by  Lemma \ref{GESbound}
\begin{equation}
\label{lem3.1}
\gamma_{1/n}(T,G)\leq 2\gamma(T,G)+O\Big[\frac{K_n\lambda_{\max}(M_G^{-1})}{n}\Big\{2L_n/b+\lambda_{\max}(M_{G}^{-1})(C_1+2C_2K_n/b)\Big\}\Big]
\end{equation}
as long as $M_G$ is positive definite for  all $G\in \{H:d_{\infty}(F_n,H)\leq C\sqrt{\frac{m\log(2/\delta)}{n}}\}$.
Provided this condition holds, it would suffice to show that
$$\gamma(T,G)\leq \gamma (T,F_n)+O\bigg[\sqrt{\frac{m\log(2/\delta)}{n}}K_n\lambda_{\max}(M_{F_n}^{-1})\Big\{2L_n/b+\lambda_{\max}(M_{G}^{-1})(C_1+2C_2K_n/b)\Big\}\bigg].$$
 This last inequality is a consequence of Lemma \ref{gGESbound}. 
\end{proof}

\subsection*{Proof of Theorem \ref{rates} }
\begin{proof}
First note that  Theorem 3.1.1 in \cite{vershynin2018} guarantees that a $p$-dimensional standard Gaussian random variables concentrates around $\sqrt{p}$. Specifically, for $Z\sim N_p(0,I)$ and with probability $1-\tau$, we have that $\|Z\|-\sqrt{p} \leq C\sqrt{\log(1/\tau)}$ for some universal constant $C$.  Applying this result to our Gaussian mechanism shows that with probability $1-\tau$
$$\|A_T(F_n)-T(F_n)\|\leq   C\gamma(T,F_n)\frac{5\sqrt{2\log(n)\log(2/\delta)} }{\varepsilon n}(\sqrt{p}+\sqrt{\log(1/\tau)}).$$
The first claimed result  follows from the above expression since by Conditions \ref{cond:boundedpsi} and \ref{cond:hessian} we have that $\gamma(T,F_n)=O(K_n)$. The second claim is verified by further noting that
\begin{align*}
A_T(F_n)-T(F)&=T(F_n)-T(F)+\gamma(T,F_n)\frac{5\sqrt{2\log(n)\log(2/\delta)} }{\varepsilon n}Z \\
&=T(F_n)-T(F)+O_p\Big(\frac{K_n\sqrt{\log(n)\log(2/\delta)} }{\varepsilon n}\Big) \\
&=T(F_n)-T(F)+o_p(1/\sqrt{n}),
\end{align*}
where the last equality leveraged the assumed scaling $\frac{K_n\sqrt{\log(n)\log(1/\delta)}}{\varepsilon \sqrt{n}}=o(1)$.

\end{proof}

\subsection*{Proof of Corollary \ref{truncation}}
\begin{proof}
It is easy to check that $T_n$ satisfies the conditions of Theorem \ref{rates} and that $T_n$ converges to the maximum likelihood M-functional.
\end{proof}

\subsection*{Proof of Theorem \ref{lowerbound2}}
\begin{proof}
First note that
\begin{align*}
\gamma_{\rho}(T,F) & =  \sup_z\frac{1}{\rho}\big\|T( (1-\rho)F+\Delta_x)-T(F)\big\| \\
&\geq\frac{1}{\rho}\big\|T( (1-\rho)F+\Delta_x)-T(F)\big\|\\
&= \frac{1}{\rho} \big\|\rho\int \IF(z;T,F)\mathrm d(\Delta_x-F)+o(\rho)\big\| \\
&= \frac{1}{\rho} \big\|\rho\IF(x;T,F)+o(\rho)\big\| \\
&\geq \big\|\IF(x;T,F)\big\|+o(1) ,
\end{align*}
where a von Mises expansion justifies the second equality and the third one follows from \eqref{IF}. Taking the supremum over $x$ in the last inequality we obtain
$$\gamma_{\rho}(T,F)\geq \gamma(T,F)+o(1). $$
The proof is completed by incorporating this result in the lower bound provided by Proposition \ref{lowerbound} below.
\end{proof}

Proposition \ref{lowerbound} is a generalization of Theorem 1 in \cite{chaudhuriandhsu2012} and it constitutes a somehow more general result than Theorem \ref{lowerbound2} since it gives a lower bound for any differentially private algorithm without restricting $T(F)$ to be an M-functional.

\begin{proposition}
\label{lowerbound}
Let $\varepsilon\in (0,\frac{\log 2}{2})$ and $\delta\in(0,\frac{\varepsilon}{17})$. Let $\mathfrak{F}$ be the family of all distributions over $\mathfrak{X}\subset\mathbb{R}^m$ and let $A$ be any $(\varepsilon,\delta)$-differentially private algorithm approximating $T(F)$, where $T:\mathfrak{F}\mapsto \mathbb R^k$ with $k\in\{1,\dots,p\}$. For all $n\in\mathbb{N}$ and $F\in\mathfrak{F}$, there exists a radius $\rho=\rho(n)=\frac{1}{n}\lceil\frac{\log 2}{2\varepsilon}\rceil$ and a distribution $G\in\mathfrak{F}$ with $d_{TV}(F,G)\leq \rho$, such that either
$$\mathbb{E}_{F_n}\mathbb{E}_{A}\Big[\|A(\mathcal{D}(F_n))-T(F)\|\Big]\geq \frac{\rho}{16}\gamma_{\rho}(T,F)$$
or
$$\mathbb{E}_{G_n}\mathbb{E}_{A}\Big[\|A(\mathcal{D}(G_n))-T(G)\|\Big]\geq \frac{\rho}{16}\gamma_{\rho}(T,F),$$
where $F_n$ and $G_n$ denote empirical distributions obtained from $F$ and $G$ respectively.
\end{proposition}
\begin{proof}
The claimed result can be established by extending to our multivariate setting the arguments provided in Theorem 1 of \cite{chaudhuriandhsu2012}. The only missing ingredient is a multivariate version of their Lemma 3 that we derive in Lemma \ref{lowerboundlemma} below.
\end{proof}

We will use the following result in the proof Lemma \ref{lowerboundlemma}.
\begin{lemma}
\label{composition}
Let $A:\mathbb R^{n\times m}\mapsto\mathbb R^k$ for $k\in\{1,\dots,p\}$ be any $(\varepsilon,\delta)$-differentially private algorithm, and let $D\in\cX^{n\times m}$ and $\mathcal{D}'\in\cX^{n\times m}$ be two data sets which differ by less than $n_0<n$ entries. Then, for any $S$
$$\PP[A(\mathcal{D})\in S]\geq e^{-n_0\varepsilon}\PP[A(\mathcal{D}')\in S]-\frac{\delta}{1-e^{-\varepsilon}}.$$
\end{lemma}
\begin{proof}
The same arguments of Lemma 2 in \cite{chaudhuriandhsu2012} apply here.
\end{proof}
\begin{lemma}
Let $D\in\cX^{n\times m}$ and $\mathcal{D}'\in\cX^{n\times m}$ be two data sets that differ in the value of at most $n_0<n$
entries. Furthermore let $A:\mathbb R^{n\times m}\mapsto\mathbb R^k$ for $k\in\{1,\dots,p\}$ be any $(\varepsilon,\delta)$-differentially private algorithm. For all $0<\gamma<1/3$, and for all $\tau, \tau'\in \RR^k$, if $n_0\leq \frac{\log(1/(2\gamma))}{\varepsilon}$ and if $\delta\leq \frac{1}{4}\gamma(1-e^{-\varepsilon})$, then
\label{lowerboundlemma}
$$\EE_{A}\Big[\|A(\mathcal{D})-\tau\|+\|A(\mathcal{D}')-\tau'\|\Big]\geq \gamma\|\tau-\tau'\|.$$
\end{lemma}
\begin{proof}
We adapt the proof of Lemma 3 in \cite{chaudhuriandhsu2012} to our setting. It suffices to construct two disjoint hyperrectangles $I$ and $I'$ such that
\begin{equation}
\label{LBlemma1}
\PP_{A}[A(\mathcal{D})\in I]+\PP_{A}[A(\mathcal{D}')\in I']\leq 2(1-\gamma)
\end{equation}
\begin{equation}
\label{LBlemma2}
\EE_{A}\Big[\|A(\mathcal{D})-\tau\|\big|A(\mathcal{D})\in I\Big]\geq \frac{1}{2}\|\tau-\tau'\| ~\mbox{ and }~ \EE_{A}\Big[\|A(\mathcal{D}')-\tau'\|\big|A(\mathcal{D}')\in I'\Big]\geq \frac{1}{2}\|\tau-\tau'\|
\end{equation}
since they imply that
\begin{align*}
 \EE_{A}&\Big[\|A(\mathcal{D})-\tau\|+\|A(\mathcal{D}')-\tau'\|\Big] \\
 & > \EE_{A}\Big[\|A(\mathcal{D})-\tau\|\big|A(\mathcal{D})\notin I\Big]\PP_{A}[A(\mathcal{D})\notin I] + \EE_{A}\Big[\|A(\mathcal{D}')-\tau'\|\big|A('D)\notin I'\Big]\PP_{A}[A(\mathcal{D}')\notin I']\\
& \geq \frac{1}{2}\|\tau-\tau'\| \Big( \PP_{A}[A(\mathcal{D})\notin I]+\PP_{A}[A(\mathcal{D}')\notin I'] \Big)\\
  &\geq \gamma \|\tau-\tau'\|.
\end{align*}
Let us now build $I$ and $I'$. Write $\tau=(\tau_1,\dots,\tau_k)$ and $\tau'=(\tau'_1,\dots,\tau_k')$. Without loss of generality assume that $\tau_j<\tau_j'$ and let $t_j=\frac{1}{2}(\tau'_j-\tau_j)$ for all $j=1,\dots,p$. Further let $I_j=(\tau_j-t_j,\tau_j+t_j)$ and $I'_j=(\tau_j'-t_j,\tau_j'+t_j)$. By construction the hyperrectangles $I:=I_1\times I_2\times \dots \times I_k$ and $I':=I'_1\times I'_2\times\dots\times I_k$ are disjoint and satisfy \eqref{LBlemma2}. It remains to show that \eqref{LBlemma1} holds. We proceed by contradiction. Suppose \eqref{LBlemma1} does not hold, then
\begin{align*}
2\gamma &> \PP_{A}\Big[A(\mathcal{D})\notin I\Big]+\PP_{A}\Big[A(\mathcal{D}')\notin I'\Big]\\
& \geq \PP_{A}\Big[A(\mathcal{D})\in I'\Big]+\PP_{A}\Big[A(\mathcal{D}')\in I\Big] \\
& \geq e^{-n_0\varepsilon}\Big(\PP_{A}\Big[A(\mathcal{D}')\in I'\Big]+\PP_{A}\Big[A(\mathcal{D})\in I\Big]\Big)-\frac{2\delta}{1-e^{-\varepsilon}} \\
&\geq e^{-n_0\varepsilon} 2(1-\gamma)-\frac{\gamma}{2}.
\end{align*}
The first inequality follows by assumption, the second from $I\cap I'=\emptyset$, the third one from Lemma \ref{composition} and the last one by assumption and $\delta \leq \frac{1}{4}\gamma(1-e^{-\varepsilon})$. Further note that the last inequality  leads to
$$e^{-n_0\varepsilon} 2(1-\gamma)-\frac{\gamma}{2}\geq 4\gamma(1-\gamma)-\frac{\gamma}{2}\geq \frac{7}{2}\gamma-4\gamma^2>2\gamma $$
for $\gamma\leq \frac{1}{3}$ since $n_0\leq \frac{\log(1/(2\gamma))}{\varepsilon}$.  This is a contradiction and therefore \eqref{LBlemma1} holds.
\end{proof}

\subsection*{Proof of Theorem \ref{thm3} }
\begin{proof}
The proof follows from the arguments of Theorem \ref{thm1} by combining Lemmas \ref{pSS}, \ref{pGES} and the results of \cite{nissimetal2007}. 

From Lemma \ref{pGES} we have that $\sqrt{\log n}\gamma(\alpha,F_n)>\tilde{\Gamma}_n$ for $n\geq C'\sqrt{m\log(2/\delta)} K_n \lambda_{\max}(M_F^{-1})\big\{1+2L_n/b+\lambda_{\max}(M_F^{-1})(C_1 +2C_2K_n/b )\big\}$. Given Lemma \ref{pSS}, it therefore remains to show that 
\begin{equation}
\label{thm4.1}
\frac{\sqrt{\log n}\gamma(\alpha,F_n)}{n} \geq C_{n,k}\Gamma_{U,n}\exp\Big(-C\xi\sqrt{ mn\log(2/\delta)}+\xi\Big).
\end{equation}
 Hence in order to show \eqref{thm4.1} it would suffice to establish that
\begin{equation}
\label{thm4.2}
2C\xi\sqrt{mn\log(2/\delta)}-2\xi-2\log\Big(\frac{nC_{n,k}\Gamma_{U,n}}{\gamma(\alpha,F_n)}\Big) \geq -\log\log(n).
\end{equation}
Letting $C_{n,k,U}=\frac{nC_{n,k}\Gamma_{U,n}}{\gamma(\alpha,F_n)}$ and since  $\xi\leq \frac{\varepsilon}{4\{p+2\log(2/\delta)\}}\leq 1$, 
the left hand side of \eqref{thm4.2} will be nonnegative if   
\begin{equation*}
\label{thm1.3}
n\geq \Big\{ 1+\frac{\log(C_{n,k,U})}{\xi}\Big\}^2\frac{1}{C^2m\log(2/\delta)}\geq \Big[ 1+\frac{4\{p+2\log(2/\delta)\}\log(C_{n,k,U})}{\varepsilon}\Big]^2\frac{1}{C^2m\log(2/\delta)}.
\end{equation*}
This last inequality holds by assumption.
\end{proof}

We introduce an analogue of the term $\Gamma_n$ used in the proof of Theorem \ref{thm1},  but in the context for level functionals, namely
\begin{equation}
\label{Gamma2}
\tilde{\Gamma}_n:=\sup\bigg\{\gamma_{1/n}(\alpha,G):~
d_\infty(F_n,G)\leq C\sqrt{\frac{m\log (2/\delta)}{n}} \bigg\}.
\end{equation} 
$\tilde\Gamma_n$ plays an important role in the analysis of our differentially private p-values. Lemmas \ref{pSS} guarantees that for large $n$ it suffices to control $\tilde\Gamma_n$ in order to bound the smooth sensitivity, while Lemma \ref{pGES} shows that $\tilde{\Gamma}_n$ is roughly of the same order as the empirical level gross-error sensitivity.

\begin{lemma}
\label{pSS}
Assume that Conditions \ref{cond:boundedpsi} and \ref{cond:hessian} hold. Then
$$\SS_{\xi}(\alpha,\mathcal{D}(F_n))\leq \max\bigg\{\frac{2\tilde{\Gamma}_n}{n},C_{n,k}\Gamma_{U,n}\exp\Big(-C\xi \sqrt{mn\log(2/\delta)}+\xi\Big)\bigg\},$$
where $C$ is as in \eqref{Gamma2},  $\Gamma_{U,n}=\sup\{\sup_{t\in[0,1]}\gamma\big(U,(1-t)G_n+tG_n'\big): ~d_H\big(\mathcal{D}(G_n),\mathcal{D}(G_n')\big)=1, ~G_n,G_n'\in\mathcal{G}_n\}$ and $C_{n,k}=2\frac{(k-1)^{(k-1)/2}e^{-(k-1)/2}}{\sqrt{n}2^{k/2}\Gamma(k/2)}$ where $\Gamma(\cdot)$ is the gamma function.
\end{lemma}
\begin{proof}
The result follows from arguments similar to those of Lemma \ref{SS}.
We will show that for any $\mathcal{D}(G_1)\in \RR^{n\times m}$ we have that
\begin{equation*}
e^{-\xi d_H(\mathcal{D}(F_n),\mathcal{D}(G_1))}\LS(\alpha,\mathcal{D}(G_1)) 
\leq \max\bigg\{\frac{2\tilde{\Gamma}_n}{n},\exp\big(-\xi(nr_n-1)\big)\bigg\},
\end{equation*}
where $r_n= C\sqrt{\frac{m\log(2/\delta)}{n}}$. For this we consider two possible cases. First suppose that $d_H(\mathcal{D}(F_n),\mathcal{D}(G_1))+1)/n > {r}_n$. Letting $G_1'$ such that $d_H\big(\mathcal{D}(G_1),\mathcal{D}(G_1')\big)=1$ and taking $\rho=1$ in Lemma \ref{intermediatepoint}, we get that $\LS\big(T,\mathcal{D}(G_1)\big)\leq C_{n,k}\Gamma_{U,n}$ since
\begin{align}
\label{pSS1}
&|\alpha(G_1)-\alpha(G_1')|\nonumber \\
 &\leq \Big|\int_0^1\int\IF(x;\alpha,(1-t)G_1+tG_1')\mathrm{d}(G_1-G_1')\mathrm{d}t\Big| \nonumber \\
 & \leq d_\infty(G_1,G_1') \sup_{t\in[0,1]}\gamma(\alpha,(1-t)G_1+tG_1')\nonumber \\
 & \leq \frac{1}{n}\sup_{t\in[0,1]}\sup_{x}\Big|H_k'(n\|U((1-t)G_1+tG_1')\|^2)2nU((1-t)G_1+tG_1')^T\IF(x;U,(1-t)G_1+tG_1') \Big| \nonumber\\
&\leq 2\sup_{z>0}\{H_k'(nz^2)z\}\sup_{t\in[0,1]}\sup_{x}\|\IF(x;U,(1-t)G_1+tG_1')\|\nonumber\\
&\leq C_{n,k}\Gamma_{U,n}
\end{align}
The last inequality used the definition of $\Gamma_{U,n} $ and $\sup_{z>0}\{H_k'(nz^2)z\}= \frac{(k-1)^{(k-1)/2}e^{-(k-1)/2}}{\sqrt{n}2^{k/2}\Gamma(k/2)}$.
Therefore 
$$e^{-\xi d_H(\mathcal{D}(F_n),\mathcal{D}(G_1))}\LS\big(T,\mathcal{D}(G_1)\big)\leq C_{n,k}\Gamma_{U,n} \exp\big(-\xi nr_n+\xi\big).$$

Suppose now that $[d_H(\mathcal{D}(F_n),\mathcal{D}(G_1))+1]/n \leq {r}_n$ and fix $\mathcal{D}(G_2)\in\RR^{n\times m}$ such that $d_H(\mathcal{D}(G_1),\mathcal{D}(G_2))=1$. Let $j\in\{1,\dots,n\}$ be the index at which $\mathcal{D}(G_1)$ and $\mathcal{D}(G_2)$ differ.  Finally let $\mathcal{D}(G_3)\in\RR^{(n-1)\times m}$ be the data set obtained by removing the $j$th element of $\mathcal{D}(G_1)$. Then by the triangle inequality
$$d_{\infty}(F_n,G_3)\leq d_{\infty}(F_n,G_1)+d_{\infty}(G_3,G_1)\leq [d_H(\mathcal{D}(F_n),\mathcal{D}(G_1)+1]/n\leq {r}_n $$
and hence $\gamma_{1/n}(\gamma,G_3)\leq \tilde{\Gamma}_n$. Therefore simple calculations show that
\begin{align*}
\|\alpha(G_1)-\alpha(G_2)\| &= \|\alpha(G_1)-\alpha(G_3)+\alpha(G_3)-\alpha(G_2)\| \nonumber\\
&\leq \frac{2}{n}\gamma_{1/n}(\alpha,G_3)\\
&\leq\frac{2\tilde{\Gamma}_n}{n}
\end{align*}
Since the bound holds for any choice of $\mathcal{D}(G_2)$, we see that $\LS(T,\mathcal{D}(G_1))\leq 2\tilde{\Gamma}_n/n$ and consequently $e^{-\xi d_H(\mathcal{D}(F_n),\mathcal{D}(G_1))}\LS(T,\mathcal{D}(G_1))\leq 2\tilde{\Gamma}_n/n$ .
\end{proof}
\begin{lemma}
\label{pGES}
Assume Conditions \ref{cond:boundedpsi}--\ref{cond:smoothness}. Then
$$\tilde{\Gamma}_n\leq 2\gamma(\alpha,F_n)+C'\sqrt{\frac{m\log(2/\delta)}{n}} K_n \lambda_{\max}(M_{F_n}^{-1})\Big\{1+2L_n/b+\lambda_{\max}(M_{F_n}^{-1})(C_1 +2C_2K_n/b )\Big\}.$$
\end{lemma}
\begin{proof}

We adapt the arguments developed for the estimation problem in Lemma \ref{GES}. By Lemma \ref{qGESbound} we have that
\begin{equation}
\label{lem6.1}
\gamma_{1/n}(\alpha,G)\leq 2\gamma(\alpha,G)+O\Big[\frac{1}{n}K_n \lambda_{\max}(M_F^{-1})\Big\{1+2L_n/b+\lambda_{\max}(M_G^{-1})(C_1 +2C_2K_n/b )\Big\}\Big]
\end{equation}
 for  all $G\in \{H:d_{\infty}(F_n,H) \leq C\sqrt{\frac{m\log(2/\delta)}{n}}\}$. 
Furthermore, it follows from Lemma \ref{gqGESbound}  that
\begin{equation}
\label{lem6.2}
\gamma(\alpha,G)\leq \gamma (\alpha,F_n)+O\Big[\sqrt{\frac{m\log(2/\delta)}{n}}K_n \lambda_{\max}(M_{F_n}^{-1})\Big\{1+L_n/b+\lambda_{\max}(M_{F_n}^{-1})(C_1 +2C_2K_n/b )\Big\}\Big].
\end{equation}
Using \eqref{lem6.2} in \eqref{lem6.1} shows the desired result.
\end{proof}

\subsection*{Proof of Theorem \ref{ratestest} }
\begin{proof}
Let's first consider the Wald functional. The proof of the first claim is very similar to that of Theorem \ref{rates}. The main difference is that $\|U(F_n)\|=O_P(\sqrt{k/n})$ and hence
\begin{align*}
\gamma(\alpha,F_n)&=\sup_{x}|2nH'_k(n\|U(F_n)\|^2)U(F_n)^T\IF(x;U,F_n)|\\
&\leq |2nH'_k(n\|U(F_n)\|^2)|\|U(F_n)\|\gamma(U,F_n)\\
&\leq |2nH'_k(n\|U(F_n)\|^2)|\|U(F_n)\|\|V^{-1/2}\|\gamma(T,F_n)\\
&\leq O_P(\sqrt{nk}K_n).
\end{align*}
For the second claim it suffices to notice that since
$\frac{2\log(n)\gamma(\alpha,F_n)}{n\varepsilon}Z\leq O_P\bigg(\frac{K_n\log(n)}{\sqrt{n/k}\varepsilon}\bigg)=o_P(1)$, a Taylor expansion of $H_k^{-1}$ yields
$$Q(F_n)=H_k^{-1}(\alpha(F_n)+o_P(1) )=H_k^{-1}(\alpha(F_n))+o_P(1 )=Q_0(F_n)+o_P(1).$$
It is easy to see that the proof for $\tilde{S}$ and is very similar. The same arguments also work for the Rao test since direct calculations show that
 $$\IF(x;Z,F_n)=\big(\frac{1}{n}\sum_{i=1}^n\dot{\Psi}(x_i;T_R(F_n))_{(2)}\big)\IF(x;T_R,F_n)+Z(T,F_n-\Delta_x).$$

\end{proof}

\subsection*{Proof of Theorem \ref{lowerbound3}}

\begin{proof}
First note that Proposition \ref{lowerbound} yields lower bounds of the form $\frac{\rho}{16}\gamma_{\rho}(\alpha,F) $
for our problem. We will simply further lower bound $\gamma_{\rho}(\alpha,F)$ in order to establish the claimed result. Writing $\rho_n=\sqrt{n}\rho=\frac{1}{\sqrt{n}}\lceil\frac{\log 2}{2\varepsilon}\rceil$ and $F_{\rho_n,n}=(1-\frac{1}{\sqrt{n}}\rho_n)F+\frac{1}{\sqrt{n}}\Delta_x$ for a fixed $x$, we see that
$$\sqrt{n}(U(F_n)-U(F_{\rho_n,n}))\to_d N(0,I_k). $$
Furthermore, let $\alpha(F_{\rho_n,n})=1-H_k(q_{1-\alpha_0};t(\rho_n))$, where $t(\rho_n)=n\|U(F_{\rho_n,n})\|^2$ and let  $b(\rho_n)=-H_k(q_{1-\alpha_0};t(\rho_n))$. Following the computations of Proposition 4 in \cite{heritierandronchetti1994} we have that
\begin{align*}
\alpha(F_{\rho_n,x})-\alpha(F)&= \rho_nb'(0)+\frac{1}{2}\rho_n^2b''(0)+o(\rho_n^2)+O(n^{-1})\\
  &=\rho_n^2\mu\|\IF(x;U,F)\|^2+o(\rho_n^2)+O(n^{-1})
\end{align*}
Using the last inequality we can see that
\begin{align*}
\gamma_{\rho}(\alpha,F) & =  \sup_{x'}\frac{1}{\rho}\big|\alpha( (1-\rho)F+\Delta_{x'})-\alpha(F)\big| \\
&\geq\frac{1}{\rho}\big|\alpha( (1-\rho)F+\Delta_x)-\alpha(F)\big|\\
&= \frac{1}{\rho} \big|\rho_n^2\mu\|\IF(x;U,F)\|^2+o(\rho^2_n)+O(n^{-1})\big| \\
&\geq \bigg\lceil \frac{\log 2}{2\varepsilon}\bigg\rceil\mu \|\IF(x;U,F)\|^2+o\bigg(\bigg\lceil \frac{\log 2}{2\varepsilon}\bigg\rceil\bigg).
\end{align*}
Taking the supremum over $x$ on the right hand side of the last expression completes the proof.
\end{proof}

\section{properties of the influence function}
\label{propertiesIF}

The influence function is a particular case of the G\^ateaux derivative which constitutes a more general notion of differentiability. We say that a  functional $T:\mathfrak{F}\to \Theta$ is G\^ateaux differentiable at $F$ if there is a linear functional $L=L_F$ such that for all $G\in \mathfrak{F}$
$$\lim_{t\to 0}\frac{T(F_t)-T(F)}{t}=L_F(G-F) $$
with $F_t=(1-t)F+tG$. In this section we derive some useful properties of the influence function and the gross-error sensitivity that we use to establish our differential privacy guarantees. The next lemma states a useful identity for relating the gross-error sensitivity to its fixed scale counterpart for M-estimators.

\begin{lemma}
\label{intermediatepoint} Let $F_t=(1-t)F+tG$, then for $\rho\in (0,1]$ we have
$$T(F_{\rho})-T(F)=\int_0^\rho\int \IF(x;T,F_t)\mathrm{d}(G-F)\mathrm{d}t.$$
\end{lemma}
\begin{proof} We reproduce the arguments of \cite{huberandronchetti2009} p. 38-39 for completeness. By construction
$$T(F_{\rho})-T(F_0)=\int_0^{\rho}
\frac{\mathrm{d}}{\mathrm{d}t}T(F_t)\mathrm{d}t,$$
where
$$\frac{\mathrm{d}}{\mathrm{d}t}T(F_t)=\lim_{h\to 0}\frac{T(F_{t+h})-T(F_t)}{h}.$$
Noting that $F_{t+h}$ can be rewritten as
$$F_{t+h}=\bigg(1-\frac{h}{1-t}\bigg)F_t+\frac{h}{1-t}G $$
we get that
$$\frac{\mathrm{d}}{\mathrm{d}t}T(F_t)=\frac{1}{1-t}\int\IF(x;T,F_t)\mathrm{d}(G-F_t)=\int\IF(x;T,F_t)\mathrm{d}(G-F). $$
\end{proof}

\subsection*{Gross-error sensitivity bounds}
The next lemmas provide a series of upper bounds relating  $\gamma_{\rho}(T,F)$, $\gamma(T,F_\rho)$ and $\gamma(T,G)$, where $T$ is an M-functional defined by the equation
$$\int \Psi(x,T(F))\mathrm{d}F=0.$$
The M-functional is  assumed to satisfy $\sup_x\|\Psi(x,\theta)\|\leq K$ and the following smoothness assumptions guaranteed by Condition \ref{cond:smoothness} in the main text.. There exist $r_1>0$, $r_2>0$, $C_1$ and $C_2>0$ such that
$$\|\EE_{F}[\dot{\Psi}(X,\theta)]-\EE_{G}[\dot{\Psi}(X,\theta)]\|\leq C_1 d_{\infty}(F,G) \mbox{ and } $$
$$\|\EE_{G}[\dot{\Psi}(X,\theta)]-\EE_{G}[\dot{\Psi}(X,T(G))]\|\leq C_2\|T(G)-\theta\| $$
whenever $d_{\infty}(F,G)\leq r_1$ and $\|\theta-T(G)\|\leq r_2$.
We will further assume that $\lambda_{\min}(M_G)\geq b>0$ for all $G$.
\begin{lemma}
\label{GESbound0}
Let $T$ be an M-functional defined by a bounded function $\Psi$ and such that $M_F=M(T,F)$ is positive definite. Then, for $\rho\in (0,1]$ we have that
$$\gamma_{\rho}(T,F)\leq 2\sup_{t\in[0,\rho]}\gamma(T,F_t),$$
where $F_t=(1-t)F+t\Delta_x$.
\end{lemma}
\begin{proof}
Lemma \ref{intermediatepoint} and direct calculations show that
\begin{align*}
\gamma_{\rho}(T,F)=& \sup_x\bigg\|\rho^{-1}\int_0^{\rho}\int\IF(x;T,F_{t})\mathrm{d}(\Delta_x-F)\mathrm{d}t\bigg\| \\
\leq & \sup_x\|\rho^{-1} \int_0^{\rho} M_{F_t}^{-1}\Psi(x,T(F_{t}))\mathrm{d}t
\|+\bigg\| \rho^{-1} \int_0^{\rho}\Big(\int M_{F_t}^{-1}\Psi(y,T(F_{t}))\mathrm{d}F(z)\Big)\mathrm{d}t\bigg\| \\
\leq & \sup_{t\in[0,\rho]}\sup_x\| M_{F_t}^{-1}\Psi(x,T(F_{t}))
\|+\sup_{t\in[0,\rho]}\bigg\| \int M_{F_t}^{-1}\Psi(x,T(F_{t}))\mathrm{d}F(z)\bigg\| \\
\leq & 2\gamma(T,F_{\rho}).
\end{align*}
\end{proof}

\begin{lemma}
\label{GESbound0'}
Under the assumptions of Lemma \ref{GESbound0}, if $\sup\|\dot{\Psi}(x,\theta)\|\leq L$ for all $\theta\in\Theta$ and $\lambda_{\max}(M_F^{-1})\rho\{C_1 +2C_2K/b \}<1$, we have that
\begin{align*}
\gamma(T,F_\rho)\leq& \gamma(T,F)+\rho K \lambda_{\max}(M_F^{-1})\Big\{2L/b+\lambda_{\max}(M_F^{-1})(C_1 +2C_2K/b )\Big\}\nonumber\\
& ~~~+O\Big[\rho^2K(C_1 +2C_2K/b )\Big\{C_1 +(K+L)C_2K/b \Big\}\Big]
\end{align*}
\end{lemma}
\begin{proof}
Simple manipulations and a first order (integral form) Taylor expansion shows that
\begin{align}
&\IF(x;T,F_\rho)-\IF(x;T,F) \nonumber \\
&=M_{F_\rho}^{-1}\Psi(x,T(F_\rho))-M_{F}^{-1}\Psi(x,T(F)) \nonumber \\
&= M_F^{-1}\{\Psi(x,T(F_\rho))-\Psi(x,T(F))\}+  (M_{F_\rho}^{-1}-M_F^{-1})\Psi(x,T(F_\rho))\nonumber \\
&=M_F^{-1}\tilde M_{\Delta_x}\{T(F_\rho)-T(F)\}+(M_{F_\rho}^{-1}-M_F^{-1})[\Psi(x,T(F))+\tilde M_{\Delta_x} \{T(F_\rho)-T(F)\}], 
 \label{lem9.1}
\end{align}
where $\tilde M_{\Delta_x}=\int_0^1\dot\Psi[x,T(F)+t\{T(F_\rho)-T(F)\}]\mathrm{d}t$. 
Therefore
\begin{align}
\|\IF(x;T,F_\rho)\|
\leq& \|\IF(x;T,F)\|+ \|M_F^{-1}\tilde M_{\Delta_x}\{T(F_\rho)-T(F)\}\|\nonumber\\
&~~~+\|(M_{F_\rho}^{-1}-M_F^{-1})\Psi(x,T(F))+\tilde M_{\Delta_x} \{T(F_\rho)-T(F)\}\|\nonumber \\
 \leq& \gamma(T,F)+L\lambda_{\max}(M_F^{-1})\|T(F_\rho)-T(F)\|\nonumber\\
 &~~~+\|M_{F_\rho}^{-1}-M_F^{-1}\|(K+L\|T(F_\rho)-T(F)\|) 
 \label{lem9.2}.
\end{align}
 Furthermore, using Neumann series we have that
\begin{equation}
\label{lem9.3}
M_{F_\rho}^{-1}= M_{F}^{-1}-M_{F}^{-1}(M_{F_\rho}-M_{F})M_{F}^{-1} +\sum_{k\geq 2}\Big(-M_F^{-1}(M_{F_\rho}-M_{F})\Big)^jM_F^{-1}.
\end{equation}
and  by Condition \ref{cond:smoothness} and Lemma \ref{GESbound0}
\begin{align}
& \|M_{F_\rho}-M_F\| \nonumber\\
&= \bigg\| \mathbb{E}_{F_{\rho}}[\dot{\Psi}(X,T(F_\rho))-\dot{\Psi}(X,T(F))] +\mathbb{E}_{F_\rho}[\dot{\Psi}(X,T(F))]-\mathbb{E}_{F}[\dot{\Psi}(X,T(F))]\bigg\|\nonumber\\
&\leq \bigg\|\mathbb{E}_{F_\rho}[\dot{\Psi}(X,T(F))]-\mathbb{E}_{F}[\dot{\Psi}(X,T(F))]\bigg\|+\bigg\| \mathbb{E}_{F_{\rho}}[\dot{\Psi}(X,T(F_\rho))-\dot{\Psi}(X,T(F))]\bigg\| \nonumber\\
&\leq C_1d_\infty(F_\rho,F)+C_2\|T(F_\rho)-T(F)\|\nonumber\\
&\leq \rho\{C_1 +C_2 \gamma_\rho(T,F)\} \nonumber \\
& \leq \rho(C_1 +2C_2K/b).
\label{lem9.4}
\end{align}
 Since $\|M_F^{-1}\|\|M_{F_\rho}-M_{F})\|\leq \lambda_{\max}(M_F^{-1})\rho\{C_1 +2C_2K/b \}<1$ we also have 
\begin{align}
\label{lem9.4b}
 \Big\|\sum_{k\geq 2}\Big(-M_F^{-1}(M_{F_\rho}-M_{F})\Big)^jM_F^{-1}\Big\|\leq & \|M_F^{-1}\|\sum_{k\geq 2}\|M_F^{-1}\|^j\|M_{F_\rho}-M_{F})\|^j\nonumber\\
\leq &  \|M_F^{-1}\|\bigg(\frac{1}{1-\|M_F^{-1}\|\|M_{F_\rho}-M_{F})\|}-1-\|M_F^{-1}\|\|M_{F_\rho}-M_{F})\|\bigg)\nonumber\\
\leq &\frac{\|M_F^{-1}\|^3\|M_{F_\rho}-M_{F})\|^2}{1-\|M_F^{-1}\|\|M_{F_\rho}-M_{F})\|} \nonumber\\
\leq & \frac{\rho^2(C_1 +2C_2K/b )^2 \lambda_{\max}^3(M_F^{-1})}{1-\lambda_{\max}(M_F^{-1})\rho(C_1 +2C_2K/b )}
\end{align}
Therefore using Lemma \ref{intermediatepoint}, \eqref{lem9.3}--\eqref{lem9.4b}  and taking the supremum over \eqref{lem9.2}, we obtain 
\begin{align}
\label{lem9.5}
&\gamma(T,F_\rho)\nonumber\\ 
&\leq \gamma(T,F)+L\lambda_{\max}(M_F^{-1})(2\rho K/b) \nonumber
\\
&~~~+ \Big\{\lambda_{\max}^2(M_F^{-1})\rho(C_1 +2C_2K/b)+\frac{\rho^2(C_1 +2C_2K/b )^2 \lambda_{\max}^3(M_F^{-1})}{1-\lambda_{\max}(M_F^{-1})\rho(C_1 +2C_2K/b )}\Big\}(K+2\rho L K/b)\nonumber\\
&=  \gamma(T,F)+\rho K \lambda_{\max}(M_F^{-1})\Big\{2L/b+\lambda_{\max}(M_F^{-1})(C_1 +2C_2K/b )\Big\}\nonumber\\
&~~~~~~+\rho^2K\lambda_{\max}^2(M_F^{-1})(C_1 +2C_2K/b )\bigg\{2L/b+\frac{\lambda_{\max}(M_F^{-1})(C_1 +2C_2K/b) }{1-\lambda_{\max}(M_F^{-1})\rho(C_1 +2C_2K/b )}\bigg\}\nonumber \\
& ~~~~~~~~~+2\rho^3b^{-1}K L\frac{\lambda_{\max}^3(M_F^{-1})(C_1 +2C_2K/b )^2}{1-\lambda_{\max}(M_F^{-1})\rho(C_1 +2C_2K/b )}
\end{align}
The desired result follows from \eqref{lem9.5}.

\end{proof}

\begin{lemma}
\label{GESbound}
Under the assumptions of Lemma \ref{GESbound0'} we have that 
\begin{align*}
\gamma_\rho(T,F)\leq 2\gamma(T,F)&+2\rho K \lambda_{\max}(M_F^{-1})\Big\{2L/b+\lambda_{\max}(M_F^{-1})(C_1 +2C_2K/b )\Big\}\nonumber \\
\\
&~~~+O\Big[\rho^2K(C_1 +2C_2K/b )\Big\{C_1 +(K+L)C_2K/b \Big\}\Big].
\end{align*}
\end{lemma}
\begin{proof}
This is a direct consequence of Lemmas \ref{GESbound0} and \ref{GESbound0'}.
\end{proof}

\begin{lemma}
\label{gGESbound}
Assume the conditions of Lemma \ref{GESbound} and let $d_\infty(F,G)\leq \rho$. Then we have that
\begin{align*}
\gamma(T,G)\leq & \gamma(T,F)+\rho K \lambda_{\max}(M_F^{-1})\Big\{2L/b+\lambda_{\max}(M_F^{-1})(C_1 +2C_2K/b )\Big\}\\
 & ~~~+O\bigg[\rho^2K(C_1 +2C_2K/b )\Big\{C_1 +(K+L)C_2/b  \Big\}\bigg].
\end{align*} 
\end{lemma}
\begin{proof}
 The proof is similar to that of Lemma \ref{GESbound0'}. First note that \begin{align}
&\IF(x;T,G)-\IF(x;T,F) \nonumber \\
&=M_{G}^{-1}\Psi(x,T(G))-M_{F}^{-1}\Psi(x,T(F)) \nonumber \\
&= M_F^{-1}\{\Psi(x,T(G))-\Psi(x,T(F))\}+  (M_{G}^{-1}-M_F^{-1})\Psi(x,T(G))\nonumber \\
&=M_F^{-1}\tilde M\{T(G)-T(F)\}+(M_{G}^{-1}-M_F^{-1})[\Psi(x,T(F))+\tilde M \{T(G)-T(F)\}], \label{lem11'.0}
\end{align}
where $\tilde M =\int_0^1\dot\Psi[x,T(F)+t\{T(G)-T(F)\}]\mathrm{d}t$. Further note that
\begin{equation}
\label{lem11'.1}
M_{G}^{-1}= M_{F}^{-1}-M_{F}^{-1}(M_{G}-M_{F})M_{F}^{-1} +\sum_{k\geq 2}\Big(-M_F^{-1}(M_{G}-M_{F})\Big)^jM_F^{-1}.
\end{equation}
and that applying  Lemma   \ref{intermediatepoint} with $\rho=1$ we have that
\begin{align}
\label{lem11'.1b}
\|T(G)-T(F)\|\leq & \Big\|\int_0^1\int\IF(x;T,F_t)\mathrm{d}(G-F)\mathrm{d}t\Big\| \nonumber \\
\leq &d_\infty(F,G) \sup_{t\in[0,1]}\gamma(T,(1-t)F+tG)\nonumber \\
\leq &\frac{\rho K}{b}
\end{align}
Therefore, by Condition \ref{cond:smoothness}
\begin{align}
 \|M_{G}-M_F\| 
&= \bigg\| \mathbb{E}_{G}[\dot{\Psi}(X,T(G))-\dot{\Psi}(X,T(F))] +\mathbb{E}_{G}[\dot{\Psi}(X,T(F))]-\mathbb{E}_{F}[\dot{\Psi}(X,T(F))]\bigg\|\nonumber\\
&\leq C_2\|T(G)-T(F)\|+C_1d_\infty(G,F)\nonumber\\
&\leq \rho(C_1 +C_2K/b)
\label{lem11'.2}
\end{align}
where the second inequality used Lemma \ref{intermediatepoint} with $t=1$. Furthermore, adapting \eqref{lem9.4b} we see that
\begin{equation}
\label{lem11'.3}
\|\sum_{k\geq 2}\Big(-M_F^{-1}(M_{G}-M_{F})\Big)^jM_F^{-1}\|\leq \frac{\rho^2(C_1 +2C_2K/b )^2 \lambda_{\max}^3(M_F^{-1})}{1-\lambda_{\max}(M_F^{-1})\rho(C_1 +2C_2K/b )}
\end{equation}
 Combining \eqref{lem11'.0}--\eqref{lem11'.3} we see that
\begin{align}
&\|\IF(x;T,G)\| \nonumber \\
&\leq \|\IF(x;T,F)\|+\|M_F^{-1}\tilde M\{T(G)-T(F)\}\|+ \|(M_{G}^{-1}-M_F^{-1})[\Psi(x,T(F))+\tilde M \{T(G)-T(F)\}\| \nonumber \\
&\leq \|\IF(x;T,F)\|+L\lambda_{\max}(M_F^{-1})\rho K/b+\|M_{G}^{-1}-M_F^{-1}\|(K+L\rho K/b) \nonumber \\
&\leq \gamma(T,F)+\rho L\lambda_{\min}(M_F)^{-1}K/b \nonumber\\
&~~~+\Big(\lambda_{\max}^2(M_F^{-1})\rho(C_1 +C_2K/b)+\frac{\rho^2(C_1 +C_2K/b )^2 \lambda_{\max}^3(M_F^{-1})}{1-\lambda_{\max}(M_F^{-1})\rho(C_1 +C_2K/b )}\Big)\Big\{K+L\rho K/b\Big\}\nonumber\\
&\leq \gamma(T,F)+\rho \lambda_{\min}(M_F)^{-1}K\Big\{L/b+   \lambda_{\max}(M_F^{-1})(C_1 +C_2K/b ) \Big\}\nonumber\\
& ~~~+\rho^2K\lambda_{\max}^2(M_F^{-1})(C_1 +C_2K/b )\Big\{L/b+\frac{\lambda_{\max}(M_F^{-1})(C_1 +C_2K/b ) }{1-\lambda_{\max}(M_F^{-1})\rho(C_1 +2C_2K/b )}\Big\} \nonumber \\
& ~~~~~~+\rho^3b^{-1}KL\frac{\lambda_{\max}(M_F^{-1})(C_1 +C_2K/b )^2 }{1-\lambda_{\max}(M_F^{-1})\rho(C_1 +2C_2K/b )}
\label{lem11'.4}
\end{align}
Therefore taking the supremum over \eqref{lem11'.4}  we obtain the desired result
\end{proof}

\subsection*{Generalized gross-error sensitivity bounds}
We now provide a series of bounds on the gross-error sensitivity of a general functional $g$ that we use to study the three test functionals $W$, $R$ and $\tilde{S}$ described in Section \ref{backgroundtests}. Our results rely on the following assumptions on $g$.
\begin{condition}
\label{testfunctional}
The function  $g:\mathbb{R}^p\times \cF \to\mathbb{R}$ has two continuous partial derivatives with respect to its two arguments.  Furthermore its first and second order partial derivatives with respect to the corresponding two arguments $\nabla_1 g$, $\nabla_2 g$ $\nabla_{11} g$ ,$\nabla_{12} g$, $\nabla_{21} g$ and $\nabla_{22} g$ are bounded in sup norm by some constant $\bar{C}$.
\end{condition}

 Lemma \ref{testconditions} shows that, under usual regularity conditions, the test functionals $W$, $R$ and $\tilde{S}$ satisfy Condition \ref{testfunctional}. Lemmas \ref{qGESbound0}--\ref{gqGESbound} provide inequalities relating different gross-error-sensitivity functions of $h(F)=g(T(F),F)$, namely $\gamma(h,F)$, $\gamma(h,F_\rho)$, $\gamma_\rho(h ,F)$ and $\gamma(h ,G)$. These results are in the spirit of Lemmas \ref{GESbound0}--\ref{gGESbound}.
\begin{lemma}
\label{testconditions}
The test functionals $W(F)$, $R(F)$ and $\tilde{S}(F)$ can be written as $g(T(F),F)$. If in addition Condition \ref{cond:boundedpsi} holds with $K_n=K<\infty$ and $L_n=L<\infty$, then $g$ satisfies Condition \ref{testfunctional}.
\end{lemma}
\begin{proof}
We verify the claims separately for  $W(F)$, $R(F)$ and $\tilde{S}(F)$ in the three points below.
\begin{enumerate}
\item Wald functional: it is immediate from the definition of $W$ that in fact $W(F)=f\circ T(F) = g(T(F),F)$, where $f:\theta\to \theta_{(2)}^T(V(T,F)_{22})^{-1}\theta_{(2)}$. Therefore  $g$ is constant function of the second argument and $\nabla_{2}g=0$, $\nabla_{21}g=0$ and $\nabla_{22}g=0$. Furthermore $g$ is quadratic in its second argument and $\nabla_1g(T(F),F)=(0^T,2\theta_{(2)}^T(V(T,F)_{22})^{-1})^T$ and $\nabla_{11}g(T(F),F)=\mbox{blockdiag}\{0,2(V(T,F)_{22})^{-1}\}$.
\item Rao functional: since $R$ is  quadratic in the functional $Z(T,F)=\int \Psi(X,T_R(F))_{(2)}\mathrm{d}F=f(T_R(F),F)$, we have that $R(F)=g(T_R(F),F)$. Hence in order to check Condition \ref{testfunctional} it suffices to see that the derivatives of $f$ are bounded because
$$\nabla_1f(T_R(F),F)=\int \dot{\Psi}(X,T_R(F))_{(2)}\mathrm{d}F, ~~~\nabla_{2}f(T_R(F),F)= \Psi(X,T_R(F))_{(2)}$$
$$\nabla_{11}f(T_R(F),F)=\int \frac{\partial}{\partial\theta}\dot{\Psi}(X,\theta)_{(2)}\mathrm{d}F\Big|_{\theta=T_R(F)} ,~~~\nabla_{12}f(T_R(F),F)= \dot{\Psi}(X,T_R(F))_{(2)} $$
and
$$ \nabla_{22}f(T_R(F),F)= 0.$$
\item Likelihood ratio-type functional: since $\tilde{S}(F)$ is quadratic in $T(F)_{(2)}$, the arguments given for $W$ apply.
\end{enumerate}

\end{proof}
\begin{lemma}
\label{qGESbound0}
Under the assumptions of Lemma \ref{GESbound0} and $h(F)=g(T(F),F)$, we have that $$\gamma_\rho(h,F)\leq 2\sup_{t\in[0,\rho]}\gamma(h,F_t),$$
where $F_t=(1-t)F+t\Delta_x$.
\end{lemma}
\begin{proof}
Since Lemma \ref{intermediatepoint} applies to $h(F_\rho)-h(F)$, the same arguments used in the proof of Lemma \ref{GESbound0} show the claimed result.
\end{proof}
\begin{lemma}
\label{qGESbound0'}
Under the assumptions of Lemma \ref{GESbound0'} and $h(F)=g(T(F),F)$, we have that
$$\gamma(h,F_\rho)\leq \gamma(h,F)+O\Big[\rho K\big\{L/b+(C_1 +2C_2K/b )\big\}\Big].$$
\end{lemma}
\begin{proof}
First note that
$$\IF(x;h ,F)=\nabla_1g(T(F),F)\IF(x;T,F)+\nabla_2g(T(F),F)(\Delta_x-F)$$
and
\begin{align}
&\IF(x;h ,F_\rho)-\IF(x;h ,F) \nonumber\\
&= \Big\{\nabla_1g(T(F_\rho),F_\rho)\IF(x;T,F_\rho)-\nabla_1g(T(F),F)\IF(x;T,F)\Big\}\nonumber\\
& ~~~~~ +\Big\{\nabla_2g(T(F_\rho),F_\rho)(\Delta_x-F_\rho)-\nabla_2g(T(F),F)(\Delta_x-F)\Big\} \nonumber\\
&=I_1+I_2.
\label{lem11.0}
\end{align}
We will proceed to bound $|I_1|$ and $|I_2|$ separately since from \eqref{lem11.0} we see that
$$|\IF(x;h ,F_\rho)|\leq |\IF(x;h ,F)|+|I_1|+|I_2|.$$
Let us first focus on $I_2$. Note that
\begin{align}
I_2&=(1-\rho)\nabla_2 g(T(F_\rho),F_\rho)(\Delta_x-F) -\nabla_2 g(T(F),F)(\Delta_x-F)\nonumber \\
&=(1-\rho)\Big\{\nabla_2 g(T(F_\rho),F_\rho)(\Delta_x-F)-\nabla_2 g(T(F),F)(\Delta_x-F)\Big\}-\rho\nabla_2 g(T(F),F)(\Delta_x-F)
\label{I2}
\end{align}
and that $\nabla_2 g(\theta,F)(\Delta_x-F) $ can be viewed as a function $g_2:\Theta\times \mathcal{F}\to \mathbb{R}$. Therefore viewing $g_2$ as a function of the first argument allows us to get a first order Taylor expansion of the form 
$$g_2(\theta,F)=g_2(\theta',F)+\Big[\int_0^1\nabla_1g_2\{\theta'+t(\theta-\theta'),F\}\mathrm{d}t\Big](\theta-\theta')=g_2(\theta',F)+(\nabla_1\bar g_2)(\theta-\theta'),$$
while viewing $g_2$ as function of its second argument leads to 
$$g_2(\theta,F_\rho)=g_2(\theta,F)+\int_0^\rho\frac{\mathrm{d}}{\mathrm{d}t}g_2(\theta,F_t)\mathrm{d}t=g_2(\theta,F)+\int_0^\rho\nabla_2g_2(\theta,F_t)(\Delta_x-F)\mathrm{d}t.$$
Applying consecutively the above expressions in the identity  \eqref{I2} yields
\begin{align}
I_2=&(1-\rho)\Big\{\nabla_2 g(T(F),F_\rho)(\Delta_x-F) -\nabla_2g(T(F),F)(\Delta_x-F)+(\nabla_1\bar{g}_2)\big(T(F_\rho)-T(F)\big)\Big\}
\nonumber \\
& ~~~~~~~~~    -\rho\nabla_2 g(T(F),F)(\Delta_x-F)\nonumber\\
 =&(1-\rho)\Big\{\int_0^\rho\nabla_2g_2(T(F),F)(\Delta_x-F)\mathrm{d}t+(\nabla_1\bar{g}_2)\big(T(F_\rho)-T(F)\big)\Big\}   -\rho\nabla_2 g(T(F),F)(\Delta_x-F)\nonumber \\
 =&(1-\rho)\Big\{\int_0^\rho\nabla_2g_2(T(F),F_t)(\Delta_x-F)\mathrm{d}t+\rho(\nabla_1\bar{g}_2)\IF_\rho(x;T,F)\Big\}   -\rho\nabla_2 g(T(F),F)(\Delta_x-F).
\label{lem11.3}
\end{align}
Since Condition \ref{testfunctional} guarantees that all the first two partial derivatives of $g$ are bounded, from  \eqref{lem11.3} and the triangle inequality we see that 
\begin{equation}
\label{lem11.4}
|I_2|=O\Big[\rho\big\{1+\gamma_\rho(T,F)\big\}\Big]
\end{equation}
Let us now study $I_1$. Note that  $\nabla_1 g(\theta,F)$ can be viewed as a function $g_1:\Theta\times \mathcal{F}\to \mathbb{R}^p$ and admits analogous expansions to the ones considered for $g_2$ in \eqref{lem11.3}. Therefore
\begin{align}
I_1&=\nabla_1g(T(F),F)\Big\{\IF(x,T,F_\rho)-\IF(x;T,F)\Big\} \nonumber\\
&~~~ +\Big\{\nabla_1g(T(F_\rho),F_\rho)-\nabla_1g(T(F),F)\Big\}\IF(x;T,F_\rho) \nonumber \\
&=\nabla_1g(T(F),F)\Big\{\IF(x,T,F_\rho)-\IF(x;T,F)\Big\}\nonumber\\
&~~~ +\Big\{\nabla_1g(T(F),F_\rho)-\nabla_1g(T(F),F)+(\nabla_{2}\bar{g}_1)\big(T(F_\rho)-T(F)\big)\Big\}\IF(x;T,F_\rho) \nonumber\\
& =\nabla_1g(T(F),F)\Big\{\IF(x,T,F_\rho)-\IF(x;T,F)\Big\}\nonumber\\
&~~~ +\bigg\{\int_0^\rho\nabla_{1}g_1(T(F),F_t)(\Delta_x-F)\mathrm{d}t+\rho(\nabla_{2}\bar{g}_1)\IF_\rho(x;T,F)\bigg\}\IF(x;T,F_\rho).
\label{lem11.1}
\end{align}
Using the Cauchy-Schwarz inequality, \eqref{lem11.1} and \eqref{lem9.1}--\eqref{lem9.4}  we see that
\begin{align}
|I_1|&\leq \|\nabla_1g(T(F),F)\|\|\IF(x;T,F_\rho)-\IF(x;T,F)\| \nonumber \\
& ~~~ +\Big\|\int_0^\rho\nabla_{1}g_1(T(F),F_t)(\Delta_x-F)\mathrm{d}t\Big\| \|\IF(x;T,F)\| +\rho\| \nabla_{2}\bar{g}_1\|\|\IF(x;T,F)\|\|\IF_\rho(x;T,F)\| \nonumber \\
& \leq \bar{C} \Big\{\|\IF(x;T,F_\rho)-\IF(x;T,F)\| +\rho \gamma(T,F)+\rho \gamma_\rho(T,F)\gamma(T,F)\Big\}\nonumber \\
&\leq \bar{C} \Big\{\Big\|(M_{F_\rho}^{-1}-M_F^{-1})[\Psi(x,T(F))+\tilde M_{\Delta_x} \{T(F_\rho)-T(F)\}]+M_F^{-1}\tilde M_{\Delta_x}\{T(F_\rho)-T(F)\}\Big\| \nonumber\\
 & ~~~~~~~~~ +\rho \big(1+\gamma_\rho(T,F)\big)\gamma(T,F) \Big\}\nonumber\\
&\leq \bar{C}\bigg[\lambda_{\max}(M_F)\Big\{\frac{\rho(C_1 +2C_2K/b)}{\lambda_{\min}(M_F)^2}+\frac{\rho^2(C_1 +2C_2K/b )^2 \lambda_{\max}^2(M_F^{-1})}{\lambda_{\max}(M_F^{-1})-\rho(C_1 +2C_2K/b )}\Big\}(K+2\rho L K/b)\nonumber \\
&  ~~~ ~~~ +L\lambda_{\max}(M_F^{-1})(2\rho K/b) +\rho L \lambda_{\min}(M_F)^{-1}\gamma_{\rho}(T,F)  +\rho \big(1+\gamma_\rho(T,F)\big)\gamma(T,F) \bigg]\nonumber\\
& \leq \bar{C}\bigg[\rho K \lambda_{\max}(M_F^{-1})\Big\{2L/b+\lambda_{\max}(M_F^{-1})(C_1 +2C_2K/b )\Big\}\nonumber\\
&~~~~~~+\rho^2K\lambda_{\max}^2(M_F^{-1})(C_1 +2C_2K/b )\bigg\{2L/b+\frac{\lambda_{\max}(M_F^{-1})(C_1 +2C_2K/b) }{1-\lambda_{\max}(M_F^{-1})\rho(C_1 +2C_2K/b )}\bigg\}\nonumber \\
& ~~~~~~~~~+2\rho^3b^{-1}K L\frac{\lambda_{\max}^3(M_F^{-1})(C_1 +2C_2K/b )^2}{1-\lambda_{\max}(M_F^{-1})\rho(C_1 +2C_2K/b )}\nonumber\\
& ~~~~~~~~~~~~~~~+\rho \big(1+2\rho K/b\big)\lambda_{\max}(M_F^{-1})K \bigg]
\label{lem11.2}
\end{align}
Using  Lemma \ref{GESbound} to further upper bound \eqref{lem11.4} and \eqref{lem11.2} and  taking the supremum over the left hand side term of the resulting inequalities yields the desired result.
\end{proof}
%
%
\begin{lemma}
\label{qGESbound}
Under the assumptions of Lemma \ref{GESbound} we have that
 \begin{equation*}
\gamma_\rho(g ,F) \leq 2\gamma(g ,F)+O\bigg[\rho K \lambda_{\max}(M_F^{-1})\Big\{1+2L/b+\lambda_{\max}(M_F^{-1})(C_1 +2C_2K/b )\Big\}\bigg].
\end{equation*}
\end{lemma}
\begin{proof}
The result is immediate from Lemmas \ref{qGESbound0} and \ref{qGESbound0'}.
\end{proof}
%
%
\begin{lemma}
\label{gqGESbound}
Under the assumptions of Lemma \ref{gGESbound} we have that
$$\gamma(g,G)\leq \gamma(g,F)+O\bigg[\rho K \lambda_{\max}(M_F^{-1})\Big\{1+L/b+\lambda_{\max}(M_F^{-1})(C_1 +2C_2K/b )\Big\}\bigg].$$
\end{lemma}
\begin{proof}
The proof is similar that of Lemma \ref{qGESbound0'} and is omitted for the sake of space. 
\end{proof}

\section{variance sensitivity of test statistics}
\label{AppendixCVF}

In this appendix we detail the consequences of estimating the standardizing matrices $M(T,F)$, $U(T,F)$ and $V(T,F)$ on our construction. For this we need the change of variance function as a complementary tool to the influence function for the analysis of the sensitivity of the test functionals.

\subsection*{The change of variance function of M-estimators}

The change of variance function of an M-functional $T$ at the model distribution $F$ is defined as
$$\CVF(x;T,F):=\frac{\partial}{\partial t} V(T,(1-t)F+t\Delta_x)\Big|_{t=0}$$
for all $x$ where this expression exists; see \cite{hampeletal1981} and \cite{hampeletal1986}. It is essentially the influence function of the asymptotic variance functional $V(T,F)$. It  reflects the impact of small amounts of contamination on the variance of the estimator $T(F_n)$ and hence on the length of the confidence intervals.  We reproduce  below the form of the change of variance functions for general M-estimators as derived in \cite{zhelonkin2013}.

For the sake of simplicity, we write $V=V(T,F)$, $\Psi=\Psi(x,T(F))=\begin{pmatrix} \Psi_1 & \Psi_2 &\dots&\Psi_p \end{pmatrix}^T$ and
$$
\frac{\partial\Psi}{\partial \theta}=\begin{pmatrix} \frac{\partial\Psi_1}{\partial \theta_1}& \frac{\partial\Psi_1}{\partial \theta_2} &\dots& \frac{\partial\Psi_1}{\partial \theta_p} \\
 \frac{\partial\Psi_2}{\partial \theta_1}& \frac{\partial\Psi_2}{\partial \theta_2} &\dots& \frac{\partial\Psi_2}{\partial \theta_p} \\
 \vdots & \vdots & \ddots&  \vdots\\
 \frac{\partial\Psi_p}{\partial \theta_1}& \frac{\partial\Psi_p}{\partial \theta_2}&\dots& \frac{\partial\Psi_p}{\partial \theta_p}
\end{pmatrix}
$$
Using this notation, the change of variance function of M-estimators is
\begin{align*}
\CVF(x;T,F)=V-&M^{-1}\Big(\int D\mathrm{d}F+\frac{\partial}{\partial\theta}\Psi\Big)V-V\Big(\int D\mathrm{d}F+\frac{\partial}{\partial\theta}\Psi\Big)M^{-1}\\
&+M^{-1}\Big(\int R\mathrm{d}F+\int R^T\mathrm{d}F+\Psi\Psi^T\Big)M^{-1}
\end{align*}
where
$$D=\bigg\{\Big(\frac{\partial}{\partial\theta}\frac{\partial}{\partial\theta_j}\Psi_k\Big)^T \IF(x;T,F)\bigg\}_{j,k=1}^p$$
and
$$R=\Big(\frac{\partial}{\partial\theta}\Psi\Big)\IF(x;T,F)\Psi^T .$$

\subsection*{Change of variance sensitivity for tests}

The following result shows how the influence function of standardized M-functionals depends on both the influence function and the change of variance function of its corresponding unstandardized M-functional.

\begin{proposition}
\label{CVFtest}
Let $T(F)$ be an M-functional, with associated asymptotic variance matrix $V(T,F)$. Then the influence function of the standardized functional $U(F)=V(T,F)^{-1/2}T(F)$ has the form
\begin{equation}
\label{CVFsensitivity}
\IF(x;U,F)=V(T,F)^{-1/2}\IF(x;T,F)- \frac{1}{2}V(T,F)^{-1/2}\CVF(x;T,F)V(T,F)^{-1}T(F).
\end{equation}
\end{proposition}
\begin{proof}
The result follows by applying the chain rule to the derivative of $U(F_t)$ with respect to $t$ with $F_t=(1-t)F+t\Delta_x$ and evaluating the resulting expression at $t=0$. Indeed the derivative of $V(T,F_t)$ is $\CVF(x;T,F)$, the derivative of $T(F_t)$ is $\IF(x;T,F)$ and $\mathrm{d}A^{-1/2}(H)=- \frac{1}{2}A^{-1/2}HA^{-1}$ for some symmetric $p$ dimensional matrix $A$ and $H\in\mathbb{R}^{p\times p}$
\end{proof}

One can use Proposition \ref{CVFtest} to get an upper bound of the gross-error sensitivity of the differentially private Wald test resulting for the construction of Section 4.  Note that if the change of variance function is bounded it suffices to use  the simpler bound based only on the influence function of $T$ as described in the main text. Assuming that $\dot{\Psi}$ and its derivatives are bounded, it suffices to multiply the first term of \eqref{CVFsensitivity} by  $\log(n)$ in order to guarantee a bound on the smooth sensitivity. This can be shown by extending the arguments developed in Appendix \ref{propertiesIF}. The same type of expansions work using the more complicated influence function \eqref{CVFsensitivity} at the expense of more tedious calculations. We could obtain results similar to Proposition \ref{CVFtest} for the standardized functionals used in the score and likelihood ratio tests. However, as long as $\dot{\Psi}$ and its derivatives are bounded, the simple bound discussed in the main paper suffices to yield differential privacy.

\section{Further discussions and  simulations}

\subsection{Competing methods}

Let us begin by making some general remarks regarding  differential privacy in practical settings.  We note that published work in the area usually have numerical illustrations with samples sizes of the order of $\sim 100'000$ for their methods to yield acceptable results; see for example  \citep{lei2011,chaudhurietal2011, sheffet2017,barrientosetal2019} among many others. It transpires from the existing literature that differential privacy is perceived as a very strong requirement that leads to very conservative analysis. As such, it also needs large sample sizes in order to give meaningful statistical results. This has sometimes been mentioned   explicitly in different contexts \citep{machnavajjhalaetal2008, rinaldoetal2012, abadietal2016} and is usually reflected in the very large sample sizes used in examples  or implicitly by assuming that the variables of interest are bounded. The latter is used in the computation of the sensitivity of the statistics being queried. One of the messages of our paper is that if we want to enforce differential privacy constraints on non robust estimators, this will inevitably require us to inject large amounts of noise to the analysis. However, estimators that are robust by construction will require less noise in order to ensure differential privacy. It is precisely because of this that our methods can outperform existing alternatives that rely on truncation strategies or apply bounds that assume that the variables are bounded. 

In the context of linear regression, there are a number of existing methods that can achieve differential privacy. One is tempted to take an off the shelve method that works for general empirical risk minimization problems based on either objective function perturbations or stochastic gradient descent algorithms e.g. \citep,{chaudhurietal2011, bassilyetal2014}.  However such methods typically require some Lipschitz constant that is unknown in practice which makes the tuning of such algorithms tricky. There are a couple of estimation methods tailored specifically for the linear regression framework.  In particular \cite{sheffet2019} uses random projections and compression via the Johnson-Lindenstrauss transform in order to achieve differential privacy. We do not include this estimator in our simulations since the reported results in Appendix E of that paper require very large sample sizes.  We restrict our comparisons to the estimator that  \cite{caietal2019} introduced for the linear regression model as it was shown to be minimax optimal and it exhibited good numerical performance.  
We note that there are less alternatives for hypothesis testing in the linear model context. Only the work of  \cite{sheffet2017} and \cite{barrientosetal2019} seem to directly target this issue. However, in both cases their algorithms require some delicate tuning for the respective random projection/compression step for the former  and for the subsampling and truncation steps for the latter. More importantly, both methods seem to require  sample sizes of the order $\sim 10'000-100'000$  to give satisfactory statistical results.

\subsection{Simulations}

We consider the following six simulation settings for the linear regression model in order to better illustrate the behavior of our method and compare it with the minimax optimal estimator of \cite{caietal2019}. 

\begin{enumerate}[(a)]
\item  The covariates are iid Bernoulli random variables with mean  $\pi=0.15$, the variance of the Gaussian error is $0.25$ and all the slope parameters are set to $\beta_1=\dots=\beta_5=1$. A similar setting was  considered in \cite{caietal2019}.
\item The covariates are iid standard normal, $\beta=(0.5,-0.25,0)^T$ and  $\sigma=1-0.5^2-0.25^5$ as considered in the simulated example of   \cite{sheffet2017}.
\item The same normal linear regression model considered in Section 5.1 in the main document.
\item Same as model (c) but  with heavy tailed errors  generated from a $t$-distribution with 4 degrees of freedom.
\item Same  as model (c) but  with heavy tailed errors  and covariates generated from a $t$-distribution with 4 degrees of freedom.
\item The contaminated linear regression model considered in Section 5.1.
\end{enumerate}

\begin{figure}[h!]
\begin{center}
    \includegraphics[width = 5in, height=5.5in]{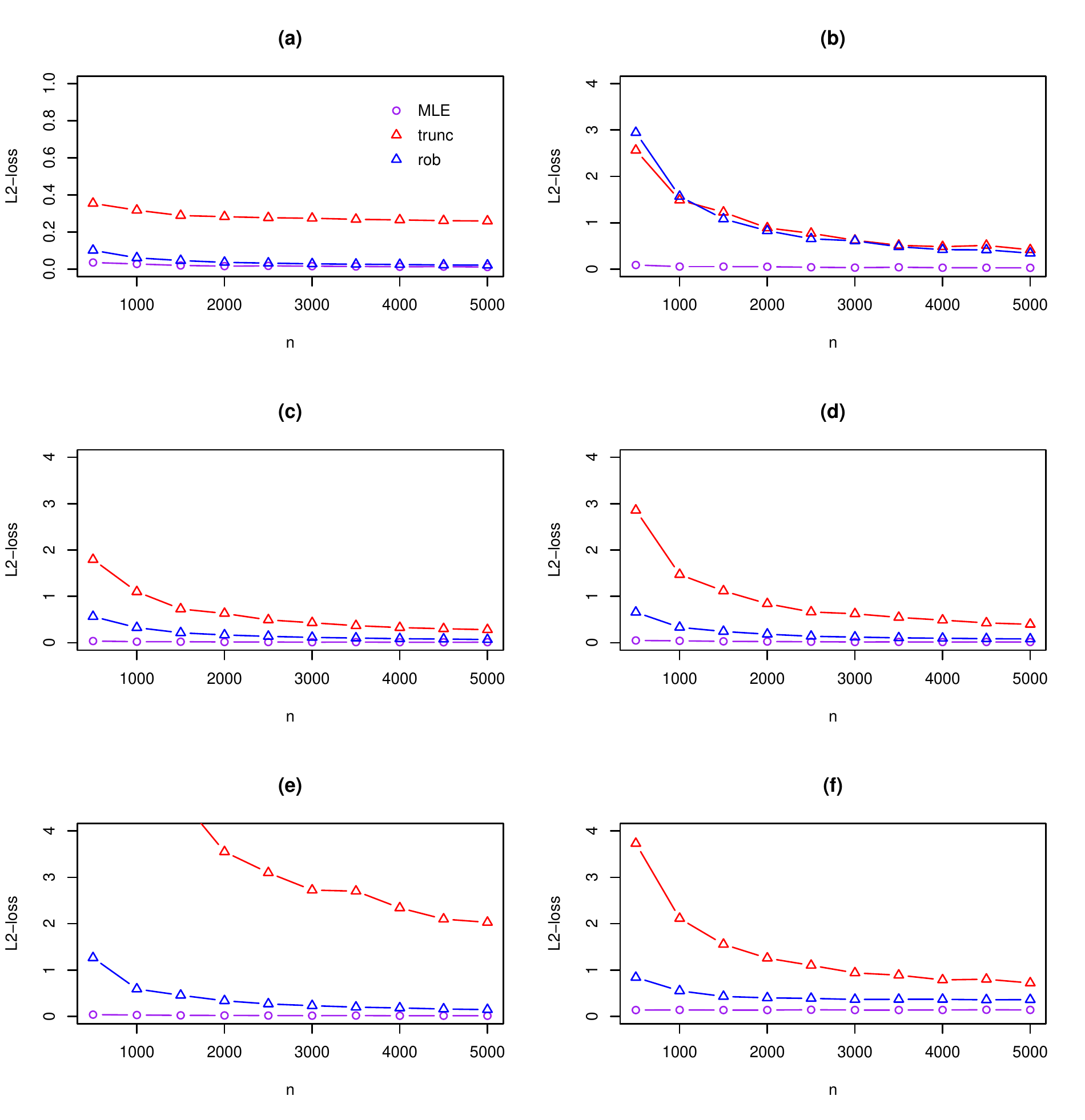}
\end{center}
\vspace{-0.5cm}
\caption{ {\small Figures (a)--(d) show the performance of the MLE, the differentially private truncated  least squares estimator and the differentially private Mallows estimator.}}
\label{Appendix}
\end{figure}%
Figure \ref{Appendix} reports the mean $L_2$  error $\|\hat\beta-\beta\|/\|\beta\|$ obtained over 100 simulations with samples sizes ranging from $n=500$ to $n=5000$ for the six settings described above.  We report the classic non private maximum likelihood estimator, the robust estimator used in Section 5.1 as well as the truncated estimator of \cite{caietal2019}.
The latter is essentially a least squares estimator for truncated responses that is rendered differentially private with the Gaussian mechanism. The level of truncation of the responses diverges as $ K \sigma\sqrt{\log n}$ for some $K>0$ and hence requires knowledge of the noise level $\sigma$.  We note that our proposal  also matches the derived optimal minimax rates of convergence up to a logarithmic term. As our simulations indicate, robust differentially private estimators can significantly outperform the behavior of the truncated least squares estimator of \cite{caietal2019} especially in the presence of heavy tails in the covariates. In this case the truncated estimator can be expected to perform poorly since it  was constructed under the assumption that the covariates are bounded.

\subsection{Assessing technical constants}
The privacy guarantees of Theorem \ref{thm1}  requires that  $n\geq \max\{N_0,N_1,N_2\}$, where  $N_1\geq \frac{1}{C^2m\log(2/\delta)}\big[ 1+\frac{4}{\varepsilon}\{p+2\log(2/\delta)\}\log\big(\frac{\lambda_{\max}(M_{F_n})}{b}\big)\big]^2 $ and $N_2 \geq (C')^2m \log(2/\delta)\big\{2L_n/b+\lambda_{\min}^{-1}(M_{F_n})(C_1+C_2\frac{K_n}{b})\big\}^2 $, for some constants $C$  and $C'$  defined in \eqref{Gamma} and Lemma \ref{GES}. Consequently,  a user that wishes to check these conditions for a given data set and an M-estimator defined by $\Psi$  needs to know the value of the constants $(C_1, C_2,N_0,b,C,C' )$. Let us therefore  focus on the evaluation of these constants. For concreteness and simplicity we focus on the robust regression with Tukey biweight loss function as it admits three continuous derivatives almost everywhere. More precisely we consider 
\[
\rho_c(t)=
\begin{cases}
1-(1-(t/c)^2)^3 & \mbox{ for } |t|\leq c \\
1 & \mbox{ for } |t|>c\end{cases}
\]
and 
\begin{equation}
\label{robLM}
\hat{\beta}=\argmin_{\beta}\sum_{i=1}^n\rho_c\Big(\frac{y_i+x_i^T\beta}{\sigma}\Big)w(x_i),
\end{equation}
where $\sigma$ is some known scale estimate and  we can chose $c=4.685$ for $95$\% efficiency at the normal model when $w(x)=1$.  For this loss function we have
$$ \rho_c'(t)=\frac{6t}{c^2}\Big(1-\frac{t^2}{c^2}\Big)^2I_{|t|\leq c}, ~, \rho_c''(t)=\frac{6}{c^2}\Big(1-\frac{t^2}{c^2}\Big)\Big(1-\frac{3t^2}{c^2}\Big)I_{|t|\leq c}, ~ \rho_c'''(t)=\Big(\frac{12t^3}{c^6}-\frac{36t}{c^4}\Big)I_{|t|\leq c},$$
 Let's first consider $C_1$ and $C_2$. Similar to $K_n$ and $L_n$ in Condition 1, their values are direct consequences of the choice of $\Psi$. Indeed, one can take $C_1=L_n$ since
 $$\|\mathbb E_{F_n-G_n}[\dot\Psi(z,\beta)]\|\leq L_n\|F_n-G_n\|_2\leq L_n d_{\infty}(F_n,G_n).$$
 Furthermore, one can take  $C_2=\max_t\rho'''(t)\lambda_{\max}(\frac{1}{n}X^TX)\|\sup_x w(x)x\|$ since for some intermediate points $\bar\beta^{(i)}$ such that $x_i^T\bar\beta^{(i)}$ lies between $x_i^T\hat\beta$ and $x_i^T\beta$, we have that
 \begin{align*}
 \|\mathbb E_{F_n}[\dot\Psi(x,y,\hat\beta)-\dot\Psi(x,y,\beta)]\|&=\Big\|\frac{1}{n}\sum_{i=1}^nx_ix_i^Tw(x_i)\rho'''_c(y_i-x_i^T\bar\beta^{(i)})x_i^T(\hat\beta-\beta)\Big\| \\
 & \leq \max_t|\rho'''_c(t)|\lambda_{\max}\Big(\frac{1}{n}X^TX\Big)\|\max_{1\leq i\leq n}w(x_i)x_i\|\|\hat\beta-\beta\| 
 \end{align*}
The minimum sample size $N_0$  defined in Condition \ref{cond:hessian} is related to the unknown minimum eigenvalue $b$ that cannot be computed in general for M-estimators.  A simple remedy of this issue is to incorporate a ridge penalty with a vanishing tuning parameter $\tau_n$ guaranteeing that $\lambda_{\min}(M_{G_n})\geq \tau_n$ for all $n$ and all empirical distributions $G_n$, hence also implying $N_0=n$. More specifically, choosing $\tau_n=1/n$ one would minimize 
\begin{equation*}
\label{ridgeRLM}
\hat{\beta}=\argmin_{\beta}\Big\{\sum_{i=1}^n\rho_c\Big(\frac{y_i+x_i^T\beta}{\sigma}\Big)w(x_i)+\frac{1}{2n}\|\beta\|^2\Big\}.
\end{equation*}
This ridge penalty would guarantee that $b\geq \frac{1}{n}$ and can be used in order to evaluate $N_1$ since the term $\log(n\lambda_{\max}(M_{F_n}))$ will remain small relative to $n$.  This approach would not lead to a meaningful way to evaluate $N_2$ and at first glance seems to suggest the sample size condition $n\geq N_2$ might be hard to meet. We note however while the term $b$ in $N_1$ comes from a worst case consideration over all empirical distributions in \eqref{lem1.1}, the term $b$ in $N_2$ was computed over all distributions $G$ such that $d_\infty(F_n,G)\leq C\sqrt{\frac{m\log(2/\delta)}{n}}$. Consequently one should think of $b$ in $N_2$ as a constant that is not too different from $\lambda_{\min}(M_{F_n})$. A non completely rigorous, but practical solution is to  replace $C_2K_n/b$ by $2C_2 \lambda_{\max}(M_{F_n}^{-1})K_n$ in the inequality defining $N_2$. Indeed, the von Mises expansion \eqref{vonMises} leads to 
$$ \|T(G)-T(F_n)\|\leq \rho \lambda_{\max}(M_{F_n}^{-1})K_n+o(\rho)$$
and 
$$\|M_G-M_{F_n}\|\leq \rho (C_1+C_2 \lambda_{\max}(M_{F_n}^{-1})K_n)+o(\rho) $$
 instead of $\|T(G)-T(F_n)\|\leq \rho K_n/b  $ and $\|M_G-M_{F_n}\|\leq \rho (C_1+C_2K_n/b)$ in \eqref{lem11'.1b} and \eqref{lem11'.2}.

The constant $C>0$ is arbitrary in Lemma \ref{SS}, but it should not be too small in order to meet the requirement  $n\geq N_1$. Similarly,  $C'>0$ should be large in Lemma \ref{GES} but  not too large in practice in order to guarantee that $n\geq N_2$. 
A closer inspection of the arguments used in proof of Lemma \ref{GES} shows that the choice of $C'$ comes from \eqref{lem9.5} and \eqref{lem11'.4}, and could be chosen to be $2C$, and one could take $C=1/\sqrt{m\log(2/\delta)}$ in order to simplify the expressions of $N_1$ and $N_2$. A much more conservative choice of $C$ would be pick a large constant that gives \eqref{Gamma} the interpreation of leading to a usual  Borel-Cantelli neighborhood around $F$. 

We note that Theorem \ref{thm3} also involves some constants $C_U$ and $C_{n,k,U}$. The former is an upper bound on the test functional used and from the arguments of Lemma \ref{testconditions}  we get that it is $2\lambda_{\max}(V(T,F)_{22})^{-1}$  for the Wald functional, $C_U=2\lambda_{\max}((T,F)_{22.1})$ for the Likelihood ratio type functional  and $C_U=\max\{L_n,L_n'\}$ for the Rao functional, where $L_n=\sup_{x}\max_{1\leq j\leq p}\|\frac{\partial^2}{\partial\theta\partial\theta^T}\Psi_j(x;\theta)\|$ and $\Psi_j$ is the $j$th component of $\Psi$. On the other hand $C_{n,k,U}=\frac{nC_n,k\Gamma_U}{\gamma(\alpha,F_n)}$ where $\frac{\gamma(\alpha,F_n)}{n}=2H_k'(n\|U_n\|^2)|U_n^T\IF(x;U,F_n)|$, and  $C_{n,k}$ and $\Gamma_U$ are defined in Lemma \ref{pSS}. 			Note  that $\Gamma_U$ can be evaluated  using \eqref{lem11'.1b}  to get the bound $\|T(G_n)-T(G_n')\|\leq \frac{K_n}{bn}$ for any $G_n,G_n'$ such that $d_H(G_n,G_n')=1$. Therefore $\Gamma_U\leq \lambda_{\max}(V(T,F)_{22})^{-1/2})\frac{K_n}{bn}$ for the Wald functional, $\Gamma_U\leq \lambda_{\max}(M(T,F)_{22.1})^{1/2})\frac{K_n}{bn}$ for the likelihood ratio type functional and $\Gamma_U\leq \lambda_{\max}(U(T,F))\frac{K_nL_n}{bn}$ for the Rao functional.

Finally, we note that in the more realistic case where $\sigma$ is unknown, one could either use a preliminary scale estimate in \eqref{robLM} or estimate it  conconmitantly with $\beta$ by solving  a system of equations similar to the one considered in Example 1. In both cases the formula of the  influence function of  $T(F_n)=\hat\beta$ becomes slightly more involved as they will now depend on the influence function of $S(F_n)=\hat\sigma$ \citep[Ch. 6.4]{huberandronchetti2009}. Consequently the assessments of the constants $(C_1, C_2,N_0,b,C,C' )$ discussed above would also need to be adapted for the estimation of $\sigma$. We leave for future research the important issue of providing a systematic treatment of evaluating such constants for  wider class of M-estimators. Not only would it render the privacy guarantees of our proposals easier to assess, it might also give some further insights into which classes of robust M-estimators could be more convenient for differential privacy.

\end{appendices}

\end{document}